\documentclass{article}

\usepackage[preprint]{neurips_2026}

\usepackage[utf8]{inputenc}
\usepackage[T1]{fontenc}

\usepackage{float}
\usepackage{caption}
\usepackage{subcaption}
\usepackage{booktabs}
\usepackage{graphicx}
\usepackage{microtype}
\usepackage{url} 

\usepackage{amsmath, amssymb, amsthm}
\usepackage{mathtools}
\usepackage{bm}
\usepackage{aliascnt}

\usepackage{xcolor}
\usepackage{tikz}
\usetikzlibrary{arrows.meta, calc, positioning}

\usepackage{algorithm}
\usepackage{algpseudocode}

\usepackage{cleveref}

\definecolor{niceTeal}{RGB}{0, 128, 128}
\definecolor{niceOrange}{RGB}{255, 140, 0}
\definecolor{nicePurple}{RGB}{128, 0, 128}

\theoremstyle{plain}
\newtheorem{theorem}{Theorem}[section]

\newaliascnt{lemma}{theorem}
\newtheorem{lemma}[lemma]{Lemma}
\aliascntresetthe{lemma}
\crefname{lemma}{lemma}{lemmas}
\Crefname{lemma}{Lemma}{Lemmas}

\newaliascnt{corollary}{theorem}
\newtheorem{corollary}[corollary]{Corollary}
\aliascntresetthe{corollary}
\crefname{corollary}{corollary}{corollaries}
\Crefname{corollary}{Corollary}{Corollaries}

\newaliascnt{definition}{theorem}
\newtheorem{definition}[definition]{Definition}
\aliascntresetthe{definition}
\crefname{definition}{definition}{definitions}
\Crefname{definition}{Definition}{Definitions}

\newaliascnt{example}{theorem}
\newtheorem{example}[example]{Example}
\aliascntresetthe{example}
\crefname{example}{example}{examples}
\Crefname{example}{Example}{Examples}

\newaliascnt{assumption}{theorem}
\newtheorem{assumption}[assumption]{Assumption}
\aliascntresetthe{assumption}
\crefname{assumption}{assumption}{assumptions}
\Crefname{assumption}{Assumption}{Assumptions}

\newcommand{\st}{\text{subject to}}

\newcommand{\R}{\mathbb{R}}

\title{Multi-ResNets for Subspace Preconditioning in Constrained Optimization}

\author{
\begin{tabular}{ccc}
Merve Karakas\thanks{Alphabetical ordering, authors contributed equally to this work.} &
Christopher J. Williams\footnotemark[1] &
Emmanuel O. Balogun \\
UCLA &
University of Oxford &
Sunrun\thanks{Work done while at Tapestry, Google.} \\
{\small\texttt{mervekarakas@ucla.edu}} &
{\small\texttt{williams@stats.ox.ac.uk}} &
{\small\texttt{ebalogun@alumni.stanford.edu}} \\[0.8em]
Sadegh Sadeghi Tabas &
Christian Brown &
Nikhil Rao \\
Tapestry, Google &
Tapestry, Google &
Tapestry, Google \\
{\small\texttt{sadeghitabas@google.com}} &
{\small\texttt{cshimizubrown@google.com}} &
{\small\texttt{raonik@google.com}}
\end{tabular}
}

\begin{document}

\maketitle

\begin{abstract}

We propose MResOpt, a staged residual neural network architecture for constrained optimization problems. Our architecture fits within predict-complete-correct pipelines and decomposes constraint satisfaction by priority through intermediate re-completion and stage-aware losses. The framework enables domain-informed ordered constraint satisfaction which allows the network to utilize ordinal structure when present. Under an idealized infinite-width regime, we show that our design behaves as sequential Gaussian Process regression.
%with a safe fallback to the highest-priority cumulative feasible set when lower-priority standalone constraints are incompatible with it.
On synthetic QP, QCQP, and SOCP benchmarks, the staged architecture improves high-priority constraint satisfaction across convex and non-convex settings. On line-flow-constrained AC optimal power flow, we introduce a physics-motivated constraint ordering and show that MResOpt supports a learned division of labor that keeps iterates on the equality manifold, achieving substantially lower high-priority violation than reprojected baselines while remaining computationally efficient.
\end{abstract}

\section{Background and Introduction}

Neural networks are used as fast surrogates for parametric constrained optimization problems, where the problem structure is fixed but parameters vary across instances \cite{donti2021dc3,nguyen2025fsnet}.
Existing approaches enforce constraints through three main paradigms. 
\textit{Penalty methods} incorporate constraint violations into the training loss with no feasibility guarantee. 
\textit{Projection-based methods} map network outputs onto the feasible set via differentiable optimization layers \cite{amos2017optnet,agrawal2019differentiable}. 
\textit{Iterative methods} integrate differentiable completion and correction steps, achieving feasibility through repeated solver-based projections.
No approach is ubiquitous; the choice is problem dependent and combination methods are evolving.
In this work, we design architectures with inductive biases that prioritize user deemed higher-importance constraints in an iterative-projection framework for problems with ordinal constraint structure.
% We study problems with ordinal constraint structure and design architectures with inductive biases that prioritize higher-importance constraints in a iterative-projection framework. 
Such structure arises naturally in physical systems--e.g., in AC optimal power flow (ACOPF), physical laws must be strictly satisfied, while operational constraints may be relaxed to preserve feasibility.

We target our iterative-projective methodology for established benchmarks, inclusive of ACOPF, set by DC3 \cite{donti2021dc3} which introduced the predict–complete–correct paradigm for constrained learning. 
DC3 has since set the general framework for methods targeting ACOPF, in which equality constraints are enforced via a completion step, while inequality constraints are handled through gradient-based correction. 
Subsequent work has highlighted limitations of this correction mechanism. 
FSNet \cite{nguyen2025fsnet}, evaluating DC3 on smooth convex QP, QCQP, and SOCP problems, reports strong sensitivity to the correction step size. 
Homeomorphic Projection \cite{liang2024hproj} further studies this correction step (termed “D-Proj”) as a post-processing baseline on ACOPF, achieving only 78.6\% feasibility on the 30-bus system despite per-instance tuning.
More broadly, several recent methods—including QCQP-Net \cite{chen2024qcqpnet} and DeepOPF-NGT \cite{wu2024deepopfngt}—note the fragility of DC3 on nonlinear problems and propose alternative mechanisms for enforcing feasibility.

% We revisit DC3 on ACOPF to establish a reliable baseline for our proposed methodology. 
Consistent with the reported instabilities, we find that DC3 admits no effective active-correction regime: weak correction is inert, while strong correction drifts off the equality manifold. 
We identify a discrepancy between the continuous-time analysis underlying DC3 and its discretized implementation, which is currently unaccounted for and has been noted in subsequent work (see Appendix for a discretization analysis). 
For our first contribution, we introduce a minimally modified baseline, DC3+recomp, which adds a projection step after each correction to account for this discretization error. 
% This modification does not strengthen the correction mechanism itself, but stabilizes it by preserving feasibility of the equality constraints.
Despite addressing this known open issue, we leverage the improvement to build stronger baselines for our primary hypothesis: we posit that neural networks with inductive biases for domain-informed ordered constraint satisfaction improve predict–complete–correct methods when constraints exhibit ordinal structure, such as the case for ACOPF.
Complementing DC3, we propose a general framework applicable to any predict–complete–correct pipeline with separable completion and correction operators.
For benchmarking, we instantiate it on DC3 as is standard and on DC3+recomp, but the same architectural decomposition applies to other methods sharing the predict-complete-correct framework. 
Our contributions that support this hypothesis are as follows. 

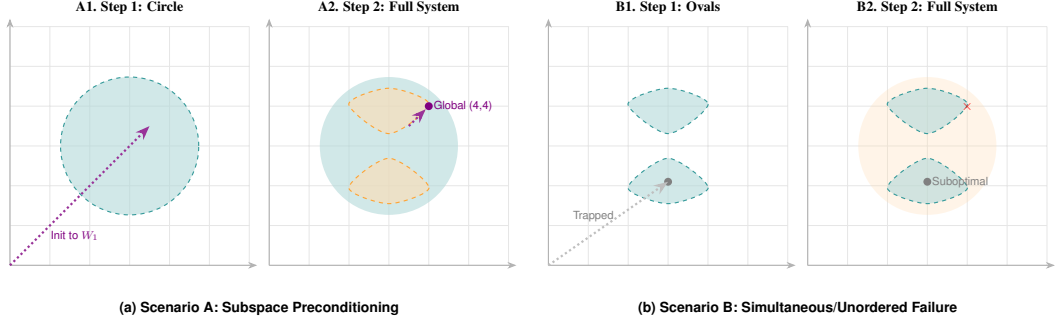
\begin{figure}
    \centering
    \resizebox{\textwidth}{!}{
        \begin{tikzpicture}[scale=1.0, >=Stealth]

            % --- Global Styles ---
            \tikzset{
                grid lines/.style={color=gray!20, thin},
                axis/.style={->, thick, color=gray!60},
                h1 region/.style={fill=niceTeal!30, draw=niceTeal!80, dashed, thick, fill opacity=0.6},
                h2 region/.style={fill=niceOrange!30, draw=niceOrange!80, dashed, thick, fill opacity=0.6},
                target pt/.style={circle, fill=nicePurple, inner sep=2pt},
                good arrow/.style={->, ultra thick, color=nicePurple!80, dotted},
                bad arrow/.style={->, ultra thick, color=gray!50, dotted},
                lbl/.style={font=\footnotesize\sffamily},
                sub caption/.style={font=\bfseries\sffamily, anchor=north, yshift=-0.8cm}
            }

            \newcommand{\drawBaseGrid}{
                \draw[grid lines] (0,0) grid (6,6);
                \draw[axis] (0,0) -- (6.2,0);
                \draw[axis] (0,0) -- (0,6.2);
            }

            % ==========================================
            % SCENARIO A: subspace preconditioning SUCCESS
            % ==========================================
            \begin{scope}[xshift=0cm]
                \node[anchor=south, font=\bfseries] at (3,6.2) {A1. Step 1: Circle};
                \drawBaseGrid
                \filldraw[h1 region] (3,3) circle ({sqrt(3)});
                \draw[good arrow] (0,0) -- (3.5,3.5);
                \node[lbl, nicePurple, anchor=south west] at (0.9,0.5) {Init to $W_1$};
            \end{scope}

            \begin{scope}[xshift=6.5cm]
                \node[anchor=south, font=\bfseries] at (3,6.2) {A2. Step 2: Full System};
                \drawBaseGrid
                \fill[h1 region, draw=none] (3,3) circle ({sqrt(3)});
                \filldraw[h2 region, smooth cycle] plot coordinates {(3, 4.45) (4.005, 4.05) (3, 3.31) (1.995, 4.05)};
                \filldraw[h2 region, smooth cycle] plot coordinates {(3, 2.69) (4.005, 1.95) (3, 1.55) (1.995, 1.95)};
                \node[target pt] (T) at (4,4) {};
                \node[lbl, right, nicePurple] at (T) {Global (4,4)};
                \draw[good arrow] (3.5,3.5) -- (T);
            \end{scope}

            % ==========================================
            % SCENARIO B: SIMULTANEOUS/BAD ORDER FAILURE
            % ==========================================
            \begin{scope}[xshift=13.5cm]
                \node[anchor=south, font=\bfseries] at (3,6.2) {B1. Step 1: Ovals};
                \drawBaseGrid
                \filldraw[h1 region, smooth cycle] plot coordinates {(3, 4.45) (4.005, 4.05) (3, 3.31) (1.995, 4.05)};
                \filldraw[h1 region, smooth cycle] plot coordinates {(3, 2.69) (4.005, 1.95) (3, 1.55) (1.995, 1.95)};
                \coordinate (Local) at (3, 2.1);
                \node[target pt, fill=gray] at (Local) {};
                \draw[bad arrow] (0,0) -- (Local);
                \node[lbl, gray, anchor=north west] at (0.5,1.5) {Trapped};
            \end{scope}

            \begin{scope}[xshift=20cm]
                \node[anchor=south, font=\bfseries] at (3,6.2) {B2. Step 2: Full System};
                \drawBaseGrid
                \fill[h2 region, draw=none, fill opacity=0.3] (3,3) circle ({sqrt(3)});
                \filldraw[h1 region, smooth cycle] plot coordinates {(3, 4.45) (4.005, 4.05) (3, 3.31) (1.995, 4.05)};
                \filldraw[h1 region, smooth cycle] plot coordinates {(3, 2.69) (4.005, 1.95) (3, 1.55) (1.995, 1.95)};
                \node[target pt, fill=gray] (LO) at (3, 2.1) {};
                \node[lbl, right, gray] at (LO) {Suboptimal};
                \node[font=\sffamily\bfseries, color=red] at (4,4) {$\times$};
            \end{scope}

            % Sub-captions
            \node[sub caption] at (6.25, 0) {(a) Scenario A: Subspace Preconditioning};
            \node[sub caption] at (19.75, 0) {(b) Scenario B: Simultaneous/Unordered Failure};

        \end{tikzpicture}
    }
    \caption{
    Visualization of \Cref{ex:basin-trap}.
    When ordinal structure is known, this can help navigate nonlinear and nonconvex through sequential subspace preconditioning through initialization. 
    Subspace preconditioning (left) navigates nonlinear and nonconvex global topology via lexicographic priority, whereas simultaneous or incorrectly ordered learning (right) succumbs to local minima entrapment. }
    \label{fig:subspace preconditioning_comparison}
\end{figure}

\paragraph{Contributions} 
We propose MResOpt (Multi-ResNet Optimization), a staged residual schema for predict–complete–correct pipelines that decomposes constraints by priority through intermediate re-completion and tier-aware losses; while the schema is general, we instantiate it on DC3. We provide an infinite-width analysis of the staged architecture, proving that detached stages behave as sequential Gaussian process regressions and that bounded activations guarantee a safe fallback to the highest-priority cumulative feasible set when lower-priority standalone constraints are incompatible with it.
Empirically, on controlled QP, QCQP, and SOCP benchmarks, the staged architecture reshapes the high-priority / low-priority / cost frontier: detach acts as a strict-priority ablation, while nodetach performs better under nonlinear coupling. On line-flow-constrained ACOPF, we diagnose that vanilla DC3 has no useful active-correction regime, with weak correction being inert and strong correction drifting off the AC equality manifold. Finally, we introduce structured baselines (DC3+manifold-aware and DC3+recomp) to isolate the effects of gradient direction and reprojection, and show that MResOpt achieves the strongest stable high-priority constraint / cost trade-off, with results generalizing to the IEEE 57-bus system (\Cref{sec:appendix-57bus}).

\paragraph{Related work}

\textit{Iterative methods and DC3 limitations.}
We focus on predict–complete–correct methods, particularly DC3 \cite{donti2021dc3}. Prior work reports brittle correction on nonlinear problems \cite{nguyen2025fsnet,liang2024hproj,chen2024qcqpnet,wu2024deepopfngt}; as discussed in the introduction, we diagnose this failure mode and introduce stabilized baselines. Our method is complementary, modifying how correction is structured across constraint priorities.
\textit{Hierarchical architectures.}
Our approach relates to multi-resolution and hierarchical networks that refine solutions from coarse to fine \cite{ronneberger2015unet,williams2023unified,falck2022multi}. Such architectures have been used in PDE solvers and cascaded refinement settings \cite{Han2023Hierarchical,DIMOLA202658,franco2023mesh,chen2017photographic}. We adapt this perspective to constrained optimization by aligning stages with a filtration over constraint sets.
\textit{Preconditioning.}
Our method is a neural analog of classical preconditioning, where structured coarse problems are solved before finer ones (e.g., multigrid, domain decomposition). Here, early stages enforce high-priority (e.g., physical) constraints, providing a stable manifold on which later stages refine lower-priority objectives.

\subsection{Predict–Complete–Correct Neural Network Proposals for Constrained Optimization} 

We seek to train a neural network $N_{\theta} : \mathbb{V} \mapsto  \mathbb{W}$, parameterized by $\theta \in \R^p$, to solve:
\begin{align}
    \min_{w \in \mathbb{W}} f(w \mid v) \quad \st \quad g(w \mid v) \leq 0, \quad h(w \mid v) = 0,
\end{align}
where $f > 0$ is a coercive function, $g : \mathbb{W} \times \mathbb{V} \to \R^{m_1}$, and $h : \mathbb{W} \times \mathbb{V} \to \R^{m_2}$. 
Here, $v \in \mathbb{V} = \mathbb{R}^n$ denotes the physical system state mapping to $\mathbb{W} = \mathbb{R}^d$; we later extend this formulation to graph domain and co-domains to make this applicable to ACOPF.
Assuming continuity of $f, g$, and $h$, the neural network $N_{\theta}$ maps the input state $v$ to a proposed optimal solution $w = N_{\theta}(v)$.

The DC3 framework \cite{donti2021dc3}, which employs a neural network $N_{\theta}: \R^n \to \R^{D}$ to propose a candidate solution $z \in \R^{D}$ with reduced dimensionality $D < d$ (e.g., $D = d - m_2$). 
This reduction constrains the networks output to match the solution manifold's degrees of freedom, preventing deviation at initialization.
Equality constraints $h$ are satisfied by completing the partial proposal $\tilde{w} = z$ via a projection operator $\mathcal{T}(z) = w$ onto the solution manifold. 
Subsequently, inequality constraints are enforced by performing gradient descent\footnote{this ensures constraint satisfaction in the continuum, but breaks upon discretization.} directly along this manifold to obtain $\hat{w} = \Pi(\tilde{w})$, as detailed in \Cref{alg:DC3train}.
The process comprises three distinct stages: a neural network proposal, numerical completion, and gradient-based correction. 
Implicit differentiation enables backpropagation through the solver, while gradient descent operates along the solution manifold to maintain equality constraints within the fine-mesh limit of the learning step size (Section 3.1 and 3.2 \cite{donti2021dc3}).

These approaches assume that generated points are both \emph{completable} and \emph{correctable}. 
These proposals must fall within a local basin of attraction of a desirable minimum; failing this, the procedure risks divergence or entrapment in poor local optima.
We aim to design neural networks with an inductive bias towards proposals $z_i$ that resolve high priority physical constraints first, yielding solutions $\hat{w}_i$ optimized for priority metrics. 
By embedding this hierarchy, we ensure proposals maintain physical validity in overdetermined regimes: even when a global optimum is unattainable, the system defaults to a physically valid state rather than an invalid compromise.

\subsection{Subspace Preconditioning Prediction Steps}
The performance of predict–complete–correct methods depends critically on the quality of the initial neural proposal. At initialization, the network’s inductive bias governs which regions of the solution space are explored; poor bias can lead to ineffective search or convergence to suboptimal basins. To address this, hierarchical architectures such as U-Nets and multi-resolution models \cite{ronneberger2015unet,DIMOLA202658,Han2023Hierarchical,chen2017photographic,franco2023mesh,falck2022multi} impose structure through a filtration of approximation spaces. These spaces encode ordinal priorities that cannot be reliably enforced through scalar loss functions in overdetermined or infeasible regimes (see \Cref{lemma:infeasible-cost}).

We adopt this perspective via a Multi-ResNet architecture \cite[Defn 3]{williams2023unified}, where a descending filtration $\{W_i\}_{i=1}^{r}$ induces a subspace-preconditioned search. The coarse space $W_1$ captures high-priority structure—e.g., inviolable physical constraints—while subsequent spaces progressively refine the solution toward the full target space $W_r$. This hierarchical bias improves optimization by guiding early iterates toward well-conditioned regions of the landscape. As illustrated in \Cref{ex:basin-trap} (\Cref{fig:subspace preconditioning_comparison}), enforcing all constraints simultaneously can lead to basin trapping in nonconvex, disconnected regions, whereas the staged approach recovers the global optimum.

In ACOPF, this ordering arises naturally from the physics. Equality constraints define the feasible manifold and are enforced exactly via completion; box bounds introduce local, smooth corrections; and line flow limits impose nonlinear, cross-variable constraints that create the non-smooth geometry responsible for basin trapping. By enforcing higher-priority constraints first, early stages operate on smoother subproblems and provide well-conditioned initializations for later refinement—analogous to coarse-to-fine strategies in multigrid and Galerkin methods. If the full feasible set is unreachable, the solution remains anchored in the most physically meaningful subspace (e.g., $W_1$). We validate the importance of this ordering via a 3-bus ablation (\Cref{sec:appendix-ordering-ablation}), where reversing the hierarchy eliminates the advantage of our method.

\section{Multi-ResNets as Subspace Preconditioners for Optimization}

\subsection{Lexicographic Constraint Importance}
Constructing the filtration requires ranking both equality and inequality constraints by priority. 
We unify notation by expressing all constraints in equality form $C(w \mid v) = 0$, where $C(w \mid v) = \max(0, g(w \mid v))$ for inequalities and $C(w \mid v) = h(w \mid v)$ for equalities. 
The target feasible set $W_r$ is embedded within the following nested structure induced by the constraints $\{C_j\}_{j=1}^n$:
\begin{align}\label{eq:filtration}
W{r} \subset W_{r-1} \subset \dots \subset W_{1} \subset W_0 = \mathbb{W},
\qquad
W_i \coloneqq \{ w \in \mathbb{W} \mid C_j(w) = 0 \text{ for all } j \leq i \},
\end{align}
where $\mathbb{W}$ denotes the ambient solution space. 
This construction yields a descending filtration $\{W_i\}_{i=0}^r$ of the search space, enabling the application of the Multi-ResNet framework \cite[Defn.~1]{williams2023unified}.
The cumulative set $W_i$ includes all constraints through priority $i$, so $W_i \subseteq W_{i-1}$. 
Define the standalone set, $S_i \coloneqq \{w \in \mathbb{W}\mid C_i(w)=0\}$, that contains only constraint $i$. 
It naturally supports a greedy sequential refinement procedure, where solutions in $W_{i-1}$ initialize the search in $W_i$.
We formalize this hierarchy in \Cref{assumption:lex-order} by positing a lexicographic ordering of constraints.

\subsection{Cumulative Feasibility and Safe Fallback under Standalone-Set Conflicts}
\label{subsec:fallback}
Lexicographic ordering provides robustness in overdetermined regimes where global feasibility may fail. 
Let $W_1$ encode the highest-priority cumulative feasible set (e.g., the physical / equality manifold in ACOPF) and let $S_2$ be the standalone set defined by a lower-priority constraint $C_2$ (e.g., operational box bounds or line-flow limits). It is possible that $W_1 \cap S_2 = \emptyset$, in which case no element of $W_1$ also satisfies $C_2$. Combining these objectives can yield solutions that satisfy neither, converging to invalid interpolations (see \Cref{lemma:infeasible-cost}).
By prioritizing $W_1$, we obtain a safe fallback to the highest-priority cumulative feasible set: solutions remain in $W_1$ even when subsequent standalone constraints are incompatible with it. In ACOPF, this guarantees adherence to physical laws rather than violating them for operational objectives (see \Cref{ex:physical-fallback}, \Cref{fig:subspace preconditioning_infeasibility}).
The nested structure of ${W_i}$ induces a monotonic refinement of feasible regions, ensuring consistency as constraints tighten (see \Cref{cor:physical-first}). We next translate this subspace preconditioning principle into a Multi-ResNet architecture that embeds this hierarchy directly into the forward pass.

\subsection{Conditional Residual Networks as Subspace Preconditioners}

To internalize the lexicographic hierarchy as an inductive bias, \cite{williams2023unified} structures the network as a composition of sequential operators. In the Multi-ResNet framework, each stage refines the previous estimate to satisfy progressively more restrictive levels of the constraint filtration $W_1 \supseteq W_2 \supseteq \dots \supseteq W_r$, formalized via a sequence of \textit{Conditional ResNets}.

\begin{definition}[ResNet, Conditional ResNet \cite{he2016deep,williams2023unified}]
Given input space $\mathbb{V}$ and solution space $\mathbb{W}$, a mapping $\mathcal{R} : \mathbb{W} \times \mathbb{V} \to \mathbb{W}$ is a \emph{ResNet} preconditioned on $\mathcal{R}_{\text{pre}} : \mathbb{V} \to \mathbb{W}$ if
\begin{align}
\mathcal{R}(w \mid v) &= \mathcal{R}_{\text{pre}}(v) + \mathcal{R}_{\text{res}}(w \mid v),
\end{align}
where $\mathcal{R}_{\text{res}} : \mathbb{W} \times \mathbb{V} \to \mathbb{W}$. A \emph{Conditional ResNet} defines a sequence $w^{(k)} = w^{(k-1)} + \mathcal{R}_{\text{res}}^{(k)}(w^{(k-1)} \mid v)$ with $w^{(1)} = \mathcal{R}_{\text{res}}^{(1)}(v)$, where each stage is conditioned on $v \in \mathbb{V}$ and $w^{(k-1)} \in \mathbb{W}$.
\end{definition}

Composing these operators yields the Multi-ResNet architecture \cite[Def.~3]{williams2023unified}, which can be viewed as a U-Net with encoder space fixed to $\mathbb{V}$ at all resolutions\footnote{set $V_i = \mathbb{V}$ for all $i$.} and decoder subspaces defined by the filtration $\{W_i\}_{i=1}^r$. Initializing the forward pass on the physical manifold $W_1$ ensures subsequent residual blocks act as corrective refinements, with $w^{(k)} = w^{(k-1)} + \mathcal{R}_{\text{res}}^{(k)}(w^{(k-1)} \mid v)$ and $w^{(k-1)}$ serving as a subspace preconditioner for stage $k$. 

This enforces the filtration through three mechanisms: (i) \textit{physical preconditioning}, where $w^{(1)} \in W_1$ satisfies inviolable constraints, anchoring the search on the physical manifold and restricting subsequent updates to feasible directions; (ii) \textit{sequential refinement}, where each $\mathcal{R}_{\text{res}}^{(k)}$ incrementally transitions $W_{k-1} \to W_k$, resolving lower-priority constraints on a conditioned subspace and avoiding unsafe compromises from simultaneous enforcement (\Cref{lemma:infeasible-cost}, \Cref{cor:physical-first}); and (iii) \textit{safe fallback}, where if stage $k$ encounters infeasible or overdetermined constraints, the iterate remains at $w^{(k-1)} \in W_{k-1}$, preserving higher-priority feasibility.

\section{MResOpt: Multi-ResNets for Optimization with Ordered Constraints}

\subsection{MResOpt and MResOpt-det Architecture and Forward Pass}

\begin{figure}[H]
    \centering
    % --- Left Side: TikZ Picture ---
    \begin{minipage}[c]{0.55\textwidth} % Increased width for diagram
        \centering
        % Scalebox shrinks the entire diagram to fit
        \scalebox{0.6}{
        \begin{tikzpicture}[
            scale=1, every node/.style={scale=1}, % Internal scale reset
            >=Stealth,
            block/.style={draw, rectangle, minimum width=1cm, minimum height=1cm, thick, font=\small, align=center},
            detach/.style={draw, rectangle, minimum width=0.8cm, minimum height=0.5cm, font=\scriptsize, fill=gray!10},
            var/.style={font=\small},
            dots/.style={font=\LARGE},
            plus/.style={draw, circle, inner sep=1pt, font=\scriptsize, thick},
            skip/.style={->, thick, dashed}
        ]
            % --- Stage 1: Initialization ---
            \node[var] (x) at (0, 0) {$v$};
            \node[block, right=0.8cm of x] (N1) {$N_1$};
            \node[var, right=0.5cm of N1] (z1) {$z_1$};
            
            % Main flow x -> N1 -> z1
            \draw[->, thick] (x) -- (N1);
            \draw[->] (N1) -- (z1);
            
            % Stage 1 Completion
            \node[var, above=0.5cm of z1] (ztilde1) {$\tilde{w}_1$};
            \node[var, right=0.4cm of ztilde1] (zhat1) {$\hat{w}_1$};
            \draw[->] (z1) -- (ztilde1) node[midway, left, font=\tiny] {$\mathcal{T}_1$};
            \draw[->] (ztilde1) -- (zhat1) node[midway, above, font=\tiny] {$\mathcal{P}_1$};
        
            % Stage 1 Detach
            \node[detach, below=0.5cm of z1] (det1) {$\text{Detach}$};
            \draw[->] (z1) -- (det1);
        
            % --- Ellipsis ---
            \node[dots, right=0.5cm of det1] (dots) {$\dots$};
            \draw[->] (det1) -- (dots);
        
            % --- Stage i: ResNet Step ---
            \node[block, right=1.0cm of dots] (Ni) {$N_i( \cdot | v)$};
            \node[plus, right=0.4cm of Ni] (sumi) {$+$};
            \node[var, right=0.4cm of sumi] (zi) {$z_i$};
            
            % Residual Path (z_{i-1} -> sum)
            \draw[->] (dots) -- +(0.2,0) |- (Ni.west); % Input z_{i-1} to Ni
            \draw[->] (dots) -- +(0.2,0) |- +(0.2,1.2) -| (sumi); % Skip z_{i-1} over block
            
            \draw[->] (Ni) -- (sumi);
            \draw[->] (sumi) -- (zi);
        
            % Stage i Completion
            \node[var, above=0.5cm of zi] (ztildei) {$\tilde{w}_i$};
            \node[var, right=0.4cm of ztildei] (zhati) {$\hat{w}_i$};
            \draw[->] (zi) -- (ztildei) node[midway, left, font=\tiny] {$\mathcal{T}_i$};
            \draw[->] (ztildei) -- (zhati) node[midway, above, font=\tiny] {$\mathcal{P}_i$};
        
            % Stage i Detach
            \node[detach, below=0.5cm of zi] (deti) {$\text{Detach}$};
            \draw[->] (zi) -- (deti);
        
            % --- CONDITIONING BUS (Skip Connections) ---
            % 1. Drop down from x (LOWERED to -2.5 as requested)
            \draw[thick] (x.south) -- ++(0,-2.5) coordinate (bus_start);
            
            % 2. Bus line across the bottom
            \draw[thick] (bus_start) -- (bus_start -| Ni) coordinate (bus_end);
            
            % 3. Dashed Arrow up into N1
            % \draw[skip] (bus_start -| N1) -- (N1.south);
            
            % 4. Dashed Arrow up into Ni
            \draw[skip] (bus_end) -- (Ni.south) node[midway, right, font=\scriptsize] {Cond.};
        
            % Labels
            \node[below=1.8cm of Ni, font=\scriptsize\bfseries] {Stage $i$ (Refinement)};
        
        \end{tikzpicture}
        }
        \captionof{figure}{DC3 Multi-ResNet Architecture, a form of U-Net, see \Cref{fig:U-multi-resnet}}
        \label{fig:subspace preconditioning_diagram}
    \end{minipage}
    \hfill
    % --- Right Side: Algorithm ---
    \begin{minipage}[c]{0.40\textwidth} % Reduced width for algo
        \vspace{0pt} % Align tops
        \hrule
        \vspace{2pt}
        \captionof{algorithm}{Forward Pass}
        \label{alg:subspace preconditioning_side}
        \vspace{-7pt}
        \hrule
        \begin{algorithmic}[1]
            \scriptsize % Use scriptsize to fit width
            \State \textbf{Input:} $v$, $\{\theta_i\}_{i=1}^r$
            \State $z_1 \gets N_{\theta_1}(v)$
            \State $\hat{w}_1 \gets \mathcal{P}_1(\mathcal{T}_1(z_1))$
            \State $\mathcal{L}_{tot} \gets \mathcal{L}_1(\hat{w}_1)$
            \For{$i = 2$ \textbf{to} $n$}
                \State \textbf{Detach} $z_{i-1}$ \hfill $\triangleright$ Skip if \Cref{sec:detach} applicable
                \State $z_i \gets z_{i-1} + N_{\theta_i}(z_{i-1} \mid v)$
                \State $\hat{w}_i \gets \mathcal{P}_i(\mathcal{T}_i(z_i))$
                \State $\mathcal{L}_{tot} \mathrel{+}= \mathcal{L}_i(\hat{w}_i)$
            \EndFor
            \State \textbf{Return} $\mathcal{L}_{tot}$
        \end{algorithmic}
        \vspace{2pt}
        \hrule
    \end{minipage}
\end{figure}

MResOpt integrates the implicit constraint solvers of DC3 with the subspace preconditioning framework of Multi-ResNets to solve constrained optimization problems while prioritizing ordinal structure.
A non-residual anchoring step, $z_1 = N_1(v)$, produces a latent proposal whose subsequent completion $\mathcal{T}_1(z_1)$ lands on the equality manifold $W_1 = \{w | C_{\mathrm{eq}}(w) = 0\}$ (in ACOPF, the AC power-flow manifold).
This establishes a valid base-manifold solution for subsequent stages to build hierarchical residuals, embedding the system's physical laws as a structural inductive bias.

For subsequent stages $i \in \{2, \dots, n\}$, the architecture implements a conditional ResNet update: $z_i = \text{Detach}(z_{i-1}) + N_i(\text{Detach}(z_{i-1}) \mid v)$ for MResOpt-det and $z_i = z_{i-1} + N_i(z_{i-1} \mid v)$ for MResOpt.
This stop-gradient operation enforces a greedy priority, preventing lower-level operational constraints from compromising the physical consistency of the foundational manifold $W_1$. 
Latent estimates $z_i$ are then mapped toward feasible space via sequence-specific completion and correction operators: $\tilde{w}_i = \mathcal{T}_i(z_i)$ and $\hat{w}_i = \mathcal{P}_i(\tilde{w}_i)$, targeting $\hat{w}_i \in W_i$.

The completion operator $\mathcal{T}_i$ generates the auxiliary variables for the $i$-th level, while the correction operator $\mathcal{P}_i$ reduces violation of $C_i$ through gradient-based updates.
Because $W_i \subseteq W_{i-1}$, each refinement targets feasibility within progressively smaller sets while the architecture is designed to preserve higher-priority feasibility. If a later stage encounters an infeasible or overdetermined constraint, the output from the previous stage provides a fallback.

\subsection{Detach and Gradient Independence}
\label{sec:detach}

In \Cref{fig:subspace preconditioning_diagram}, we employ a detach operator to prevent gradient flow from subspace $W_i$ to $W_{i+1}$. As in U-Net–style constructions \cite{williams2023unified}, this is appropriate when the subproblem on $W_{i+1} \setminus W_i$ is independent of the preconditioned problem on $W_i$. For example, in PDE settings with Multi-ResNets, Galerkin approximations yield $L^2$-orthogonality between subspaces, justifying such decoupling. Analogously, this assumption holds in convex settings where constraint priorities do not interact.

In more general (e.g., nonconvex) problems, however, solving on $W_{i+1}$ may require information from $W_i$, making gradient communication between stages beneficial. This is reflected in our experiments: detach is effective in convex regimes, promoting modularity and stabilizing training, whereas in nonconvex settings it is less effective. In these cases, removing detach—while retaining subspace preconditioning—allows dependencies between stages to be learned and improves performance. Accordingly, we evaluate both variants: MResOpt (no detach) and MResOpt-det (with detach).

From a theoretical perspective, our infinite-width analysis applies only to the detached architecture, as cross-stage gradient flow breaks the independence required for analysis. We therefore analyze MResOpt-det in the infinite-width limit, showing that it preserves the intended inductive bias and ensures sequential convergence across the lexicographic hierarchy. While this analysis assumes a static decomposition across subspaces, the same intuition carries over to the non-detached setting, where preconditioning is updated online through gradient communication between stages rather than fixed a priori.

\subsection{Gaussian Process Approximation of MResOpt-det}

To quantify the inductive bias of MResOpt, we analyze the architecture in the infinite-width limit. 
We show that MResOpt-det behaves like sequential preconditioning training (\Cref{alg:msrl}) through the subspaces $\{ W_k\}_{k=0}^{r}$ when just using \Cref{alg:subspace preconditioning_side}.
We adopt the standard \textit{Neural Tangent Kernel (NTK) parameterization}, where layer widths $\mathsf{w}_l \to \infty$. 
In this regime, the network $N_{\theta}$ converges in distribution to a Gaussian Process (GP) at initialization \cite{lee2018deep,matthews2018gaussian}. 
Furthermore, under gradient descent with squared loss, the network evolves in a \emph{lazy training} regime in which the tangent kernel remains approximately constant \cite{chizat2019lazy}, rendering the training dynamics linear and equivalent to kernel regression. 
This equivalence allows us to treat the recursive subspace preconditioning updates as a sequence of deterministic kernel mappings. 
We provide the full derivation in \Cref{appendix:GP-Approximation}.
Crucially, this approximation extends to Graph Neural Networks (GNNs), where the domains $\mathbb{V}$ and $\mathbb{W}$ are graph-structured, as detailed in \Cref{appendix:GNN-GP-Approximation}. 

\subsection{ResNet Initialization and Multi-Stage Training}

In the infinite-width limit, a ResNet can be interpreted as a mean-shifted Gaussian prior. 
A standard neural network at initialization produces a proposal $N_{\theta}(v) = w(v) \sim \mathcal{N}(0, \sigma^2)$, where the variance $\sigma^2$ depends on network depth and initialization. 
In the context of subspace preconditioning, if $w^k(v) \in W_k$ denotes a solution at level $k$, we instead obtain a proposal for level $k+1$ of the form $w_{k+1}(v) \sim \mathcal{N}(w_k(v), \sigma^2)$. 
This intuition is formalized in \Cref{lemma:ResNetGP}, which shows that a ResNet with a skip connection behaves as a Gaussian Process centered at its preconditioning mapping.
\begin{lemma}\label{lemma:ResNetGP}
Let $\mathcal{R}: \mathbb{V} \to  \mathbb{W}$ be a ResNet defined by $\mathcal{R} = \mu + \mathcal{R}_{\text{res}}$, where $\mu:  \mathbb{V} \to  \mathbb{W}$ is a deterministic measurable mapping and $\mathcal{R}_{\text{res}}$ is a neural network with hidden layer widths $n \to \infty$ under NTK parameterization. 
Then $\mathcal{R} \sim \mathcal{GP}(\mu, K)$, where $K$ is the NNGP kernel of $\mathcal{R}_{\text{res}}$.
\end{lemma}
Our staged training paradigm follows \cite[Alg.~1]{williams2023unified}, treating each transition $W_k \to W_{k+1}$ as a sequential refinement of residual error. 
By detaching prior stages during optimization, each residual block $\mathcal{R}_{\text{res},k}$ isolates the correction of remaining constraint violations while preserving the anchoring of the base equality manifold $W_1$ (the highest-priority cumulative feasible set).
In the infinite-width limit, each stage corresponds to Gaussian Process regression centered at the mean function of the previous stage.
\begin{algorithm}[h]
\caption{Multi-stage Preconditioning \cite[Alg.~1]{williams2023unified}}
\label{alg:msrl}
\begin{algorithmic}[1]
\State \textbf{In:} $( \mathcal{V} , \mathcal{W})$, prior $\mu$, stages $K$
\State $\hat{w}_0 \leftarrow \mu(\mathcal{V})$ \hfill $\triangleright$ Anchor $W_1$
\For{$k = 1 \dots K$}
    \State $E \leftarrow \mathcal{W} - \hat{w}_{k-1}$ ; $\theta_k \leftarrow \text{init}$
    \State $\theta_k \leftarrow \text{optimize } \mathcal{L}_{k}(E - \mathcal{R}_{\text{res},k}(\text{Detach}(\hat{w}_{k-1}))$
    \State $\hat{w}_k \leftarrow \hat{w}_{k-1} + \mathcal{R}_{\text{res}, k}(\text{Detach}(\hat{w}_{k-1}), \mathcal{V}; \theta_k)$
\EndFor
\State \textbf{Out:} $\mathcal{F}_K = \mathcal{R}_K \circ \dots \circ \mathcal{R}_1$ 
\end{algorithmic}
\end{algorithm}
The Multi-ResNet's staged training paradigm strictly implements the lexicographical priority proven in \Cref{lemma:infeasible-cost}.
This greedy approach preserves the structural fallback: if stage $k$ is infeasible, the physically valid $w^{(k-1)}$ remains available.
Analytically, this aligns our constraint hierarchy with the U-Net's spectral filtration: $W_1$ resolves global, low-frequency topological constraints (e.g., power balance) as a coarse-grid preconditioner, while subsequent stages $W_{k>1}$ refine high-frequency local details (e.g., voltage limits).

\subsection{Staged Training Multi-ResNet Gradient Flow}

Staged training dynamics can be modeled via the Neural Tangent Kernel (NTK) as a sequence of independent gradient flows. 
We show that the architecture in \Cref{fig:subspace preconditioning_diagram} replicates the training dynamics of \Cref{alg:msrl}. 
This approach integrates seamlessly with standard training, avoiding convergence checks for model sequences and simplifying training when filtration $\{W_k\}_{k=1}^{r}$ is high-dimensional, albeit more computationally expensive than staged training.
Since stage $k$ conditions on the detached previous output, $\theta_k$ optimization occurs in a tangent space decoupled from $\theta_{j < k}$. 
In the infinite-width limit, residual blocks act linearly near initialization. 
Thus, output $w_k(t)$ follows a linear differential equation governed by local kernel $\Theta_k$: $\frac{dw_k(t)}{dt} = \eta \Theta_k (\mathcal{W} - w_k(t))$. 
Integrating yields the update $w_k(t) = w_{k-1} + (I - e^{-\eta \Theta_k t})(\mathcal{W} - w_{k-1})$.
Crucially, as test point $v^*$ leaves the training distribution, kernel correlation $\Theta(v^*, \mathcal{V})$ vanishes. 
Consequently, for out-of-distribution inputs, the trained ResNet reverts to its initial GP prior: $\mathcal{R}_{\infty}(v^*) \to \mathcal{GP}(\mu, K)$. 

\begin{lemma}[Vanishing Influence of Training Data]
\label{lemma:vanishing-influence}
Let $\mathcal{V} \subset \mathbb{V}$ be a training set and $\mathcal{R}_{\infty}$ denote a ResNet trained to convergence under squared loss in the infinite-width limit. 
If $v^* \in \mathbb{V}$ is a test point such that the kernel correlation satisfies $\lim_{\|v^* - \mathcal{V}\| \to \infty} \|\Theta(v^*, \mathcal{V})\|_2 = 0$, then the predictive distribution at $v^*$ converges to the initial prior in distribution: $\mathcal{R}_{\infty}(v^*) \rightarrow \mathcal{GP}(\mu, K)$.
% Specifically, $\mathbb{E}[\mathcal{R}_{\infty}(v^*)] = \mu(v^*)$ and $\mathrm{Var}(\mathcal{R}_{\infty}(v^*)) = K(v^*, v^*)$.
\end{lemma}

While \Cref{lemma:vanishing-influence} establishes the statistical behavior of the network given vanishing correlation, it remains to be shown that this condition is inherent to our specific architectural choices. 
\Cref{lemma:asymptotic-decay} confirms that standard bounded activation functions provide the necessary structural decay to trigger this fallback mechanism.
By combining the statistical guarantee of \Cref{lemma:vanishing-influence} with the structural property verified in \Cref{lemma:asymptotic-decay}, we can now characterize the full training dynamics. 
The vanishing correlation decouples the residual stages in the overdetermined regime, leading to \Cref{thm:msrl-equiv}. 

\begin{theorem}[Equivalence to Multi-stage Preconditioning]
\label{thm:msrl-equiv}
Let $\mathcal{F} = \mathcal{R}_K \circ \dots \circ \mathcal{R}_1$ be a composite model of $K$ stacked ResNets. 
In the infinite-width limit, the training dynamics of $\mathcal{F}$ are equivalent to a sequential residual learning regime up to an error $O(\delta)$, provided: (i) \emph{Gradient Detachment} holds such that $\nabla_{\theta_j} \mathcal{R}_k = 0$ for all $j < k$; (ii) \emph{Kernel Positivity} ensures each local NTK $\Theta_k$ satisfies $\lambda_{\min}(\Theta_k) \ge \lambda > 0$; and (iii) \emph{Bounded Interference} ensures cross-correlation between successive manifolds satisfies $\|\Theta_k(w_{k-1}, w_{k-2})\|_2 \le \delta$.
\end{theorem}

\section{Experiments}

We evaluate MResOpt on: (i) a synthetic basin-trapping benchmark testing avoidance of local minima, (ii) an overconstrained scenario testing the safe fallback guarantee, (iii) synthetic QP/QCQP/SOCP problems isolating staged correction with and without detach, and (iv) IEEE 30-bus ACOPF under varying congestion, testing whether MResOpt enables stable active correction where DC3 fails. 
All use two inequality-correction stages on top of equality completion, though the framework supports finer hierarchies when warranted. 
For ACOPF specifically, equality is enforced exactly by Newton power-flow completion as a base manifold and is not counted as a correction stage; Stage~1 corrects Tier-1 (box bounds) and Stage~2 corrects Tier-2 (line flows), as formalized in \Cref{assumption:lex-order}.

\subsection{Synthetic Results}
We evaluate on synthetic constrained optimization problems, a controlled setting. 
Each problem has $n=200$ inputs, $d=100$ decision variables, $m_{\text{eq}}=50$ equality constraints (via completion), and $m_{\text{ineq}}=100$ main inequalities plus $2d=400$ box bounds.
Inequalities are split into two tiers: Tier-1 includes the first $50$ main inequalities and all box bounds; Tier-2 includes the remaining $50$. 
We compare DC3, MResOpt (no detach), and MResOpt-det (with detach) across QP, QCQP, and SOCP classes (see \Cref{sec:appendix-synthetic-formulations,sec:appendix-implementation}). 
Results are reported as mean $\pm$ std.\ over $4$ seeds, using $10$ correction steps and $500$ training epochs. 
We report optimality gap (OG), defined as the percentage increase over the solver optimum, and Tier-1/Tier-2 constraint violations in per-unit (p.u.) norm.

% We evaluate on synthetic constrained optimization problems to isolate the effect of subspace preconditioning in a controlled setting.
% Each problem has $n = 200$ input parameters, $d = 100$ decision variables, $m_{\text{eq}} = 50$ equality constraints (enforced by completion), and $m_{\text{ineq}} = 100$ main inequality constraints plus $2d = 400$ box bounds.
% The inequalities are split into two tiers: Tier-1 contains the first 50 main inequalities plus all box bounds, and Tier-2 contains the remaining 50 main inequalities.
% We compare DC3, MResOpt (no detach), and MResOpt-det (with detach) across QP, QCQP, and SOCP problem classes (formulations in \Cref{sec:appendix-synthetic-formulations}; implementation details in \Cref{sec:appendix-implementation}).
% Performance is reported as mean $\pm$ std.\ over 4 seeds, with 10 correction steps and 500 training epochs.
% We report the optimality gap (OG), defined as the percentage increase in objective value relative to the solver optimum, and Tier-1/Tier-2 constraint violations in per-unit (p.u.) norm.

\subsubsection{Convex Constraints}

\begin{table}[H]
\centering
\caption{Convex synthetic results (4 seeds, 10 correction steps). Bold = best, underline = second best.}
\label{tab:convex_synthetic}
\begin{tabular}{llccc}
\hline
Problem & Method & OG (\%) & Tier-1 (p.u.) & Tier-2 (p.u.) \\ \hline
\textbf{QP} & DC3 & $+0.92 \pm 0.36$ & \underline{$0.0530 \pm 0.044$} & \underline{$0.1885 \pm 0.125$} \\
 & MResOpt & $+1.98 \pm 0.96$ & $0.0582 \pm 0.084$ & $\mathbf{0.1063 \pm 0.123}$ \\
 & MResOpt-det & $+7.17 \pm 0.88$ & $\mathbf{0.0005 \pm 0.000}$ & $0.2035 \pm 0.125$ \\
\textbf{QCQP} & DC3 & $+1.17 \pm 0.52$ & $0.0544 \pm 0.038$ & $0.1044 \pm 0.111$ \\
 & MResOpt & $+2.68 \pm 0.57$ & $\mathbf{0.0185 \pm 0.020}$ & \underline{$0.0524 \pm 0.013$} \\
 & MResOpt-det & $+5.65 \pm 0.64$ & \underline{$0.0429 \pm 0.061$} & $\mathbf{0.0438 \pm 0.038}$ \\
\textbf{SOCP} & DC3 & $+0.81 \pm 0.32$ & $0.0326 \pm 0.032$ & \underline{$0.0707 \pm 0.086$} \\
 & MResOpt & $+2.07 \pm 0.99$ & \underline{$0.0074 \pm 0.012$} & $0.0786 \pm 0.127$ \\
 & MResOpt-det & $+2.57 \pm 0.92$ & $\mathbf{0.0066 \pm 0.006}$ & $\mathbf{0.0287 \pm 0.018}$ \\ \hline
\end{tabular}
\end{table}

MResOpt-det achieves near-zero Tier-1 on QP (100$\times$ reduction vs.\ DC3), confirming that detachment enables strict priority enforcement. 
On QCQP and SOCP, MResOpt (no detach) achieves the best Tier-1, while MResOpt-det performs best on Tier-2, with both variants outperforming DC3.

\subsubsection{Nonconvex Constraints}
\begin{table}[H]
\centering
\caption{Nonconvex synthetic results (4 seeds, 10 correction steps). Bold = best, underline = second best.}
\label{tab:nonconvex_synthetic}
\begin{tabular}{llccc}
\hline
Problem & Method & OG (\%) & Tier-1 (p.u.) & Tier-2 (p.u.) \\ \hline
\textbf{QP} & DC3 & $+0.56 \pm 0.13$ & $0.0343 \pm 0.021$ & $0.0356 \pm 0.002$ \\
 & MResOpt & $+1.16 \pm 0.14$ & \underline{$0.0059 \pm 0.002$} & \underline{$0.0248 \pm 0.010$} \\
 & MResOpt-det & $+1.35 \pm 0.23$ & $\mathbf{0.0054 \pm 0.001}$ & $\mathbf{0.0136 \pm 0.006}$ \\
\textbf{QCQP} & DC3 & $+5.75 \pm 5.89$ & $0.0896 \pm 0.060$ & \underline{$0.0695 \pm 0.038$} \\
 & MResOpt & $+6.98 \pm 5.93$ & $\mathbf{0.0291 \pm 0.019}$ & $\mathbf{0.0428 \pm 0.021}$ \\
 & MResOpt-det & $+1.61 \pm 0.29$ & $0.1765 \pm 0.141$ & $0.0757 \pm 0.081$ \\ \hline
\end{tabular}
\end{table}
MResOpt-det wins both tiers ($6\times$ T1 reduction, $2.6\times$ T2 reduction vs DC3).
On nonconvex QCQP, however, MResOpt-det fails (T1 = 0.18, worse than DC3), while MResOpt without detach achieves the best T1 ($3\times$ reduction) and T2 ($1.6\times$).
This demonstrates that detach is effective when constraints decouple cleanly, but fails when they are tightly coupled through nonlinear interactions --- in such cases, Stage~2 requires gradient feedback from Stage~1 to navigate the constraint landscape.

\subsection{AC Optimal Power Flow}

We extend our evaluation to the IEEE 30-bus ACOPF problem ($d = 72$ variables, $m_{\text{eq}} = 60$ equality constraints, up to 82 line flow inequalities), using 10{,}000 samples (7{,}000 train / 1{,}000 val / 2{,}000 test) generated by varying load demands (full formulation in \Cref{sec:acopf-formulation}; implementation details in \Cref{sec:appendix-implementation}).
The constraints decompose into three groups: $C_{\text{eq}}$ (AC power-balance equations), $C_{\text{box}}$ (generator and voltage box bounds), and $C_{\text{flow}}$ (thermal line flow limits), inducing the nested feasible sets $W_1 = \{w : C_{\text{eq}} = 0\} \supseteq W_2 = W_1 \cap \{C_{\text{box}} \leq 0\} \supseteq W_3 = W_2 \cap \{C_{\text{flow}} \leq 0\}$ (see \Cref{assumption:lex-order}).
The ordering places simpler, per-variable constraints first, giving Stage~1 a smooth correction landscape before Stage~2 tackles the nonlinear line flows; we validate this choice via an ordering ablation in \Cref{sec:appendix-ordering-ablation}.
The neural network predicts partial variables ($P_g^{\text{PV}}, V_m^{\text{slack} \cup \text{PV}}$); Newton power flow completes the remaining variables to satisfy $C_{\text{eq}}$ exactly.

We compare three methods:
\textbf{DC3} applies completion followed by gradient-based inequality correction in a single pass.
\textbf{DC3+recomp} modifies DC3 by re-projecting onto the equality manifold (via Newton power flow) after each correction step, testing whether manifold preservation alone explains MResOpt's advantage.
\textbf{MResOpt} uses the staged architecture: Stage~1 corrects Tier-1, then a residual network adds a learned update, re-completes via Newton power flow, and Stage~2 corrects Tier-2.
All methods use $\eta_{\text{corr}} = 10^{-4}$, 10 correction steps, and $\gamma = 0.01$; results are mean $\pm$ std over 3 seeds.

We evaluate under three congestion regimes by scaling the thermal line limits by a factor $\alpha_S \in \{0.5, 0.7, 1.0\}$, where $\alpha_S = 1.0$ corresponds to the PGLib \cite{babaeinejadsarookolaee2019pglib} default ratings and smaller values simulate increasing grid congestion.
At $\alpha_S = 1.0$, 93\% of test samples are $W_3$-feasible and we report the optimality gap (OG) relative to the Ipopt solver optimum.
At $\alpha_S = 0.7$ and $0.5$, most or all samples are infeasible ($W_3 = \emptyset$), so no true constrained optimum exists; we report raw generation cost instead of OG.
\textbf{DC3 is excluded} from bold/underline as it has substantial equality drift.

\begin{table}[H]
\centering
\caption{ACOPF $\alpha_S{=}1.0$ (93\% $W_3$-feasible, OG vs Ipopt optimum, 3 seeds). }
\label{tab:acopf_10}
\begin{tabular}{lcccc}
\hline
Method & OG (\%) & Equality (p.u.) & Tier-1 (p.u.) & Tier-2 (p.u.) \\ \hline
DC3 & $+0.00 \pm 0.02$ & $0.0098 \pm 0.0001$ & $0.0043 \pm 0.0001$ & $0.0190 \pm 0.001$ \\
DC3+recomp & $+0.25 \pm 0.02$ & $9.8\text{e-}10 \pm 2\text{e-}10$ & \underline{$0.0007 \pm 0.0001$} & $0.0190 \pm 0.001$ \\
MResOpt & $+1.19 \pm 0.07$ & $1.95\text{e-}09 \pm 4\text{e-}10$ & $\mathbf{0.0001 \pm 0.0000}$ & $\mathbf{0.0087 \pm 0.001}$ \\ \hline
\end{tabular}
\end{table}

\begin{table}[H]
\centering
\caption{ACOPF $\alpha_S{=}0.7$ (29\% $W_3$-feasible, cost reported instead of OG, 3 seeds). }
\label{tab:acopf_07}
\begin{tabular}{lcccc}
\hline
Method & Cost (\$/hr) & Equality (p.u.) & Tier-1 (p.u.) & Tier-2 (p.u.) \\ \hline
DC3 & 66.93 & $0.0501 \pm 0.001$ & $0.0071 \pm 0.0002$ & $0.2991 \pm 0.006$ \\
DC3+recomp & 67.27 & $1.3\text{e-}08 \pm 4\text{e-}10$ & \underline{$0.0036 \pm 0.0002$} & \underline{$0.3163 \pm 0.007$} \\
MResOpt & 67.55 & $6.5\text{e-}08 \pm 5\text{e-}09$ & $\mathbf{0.0016 \pm 0.0001}$ & $\mathbf{0.3127 \pm 0.007}$ \\ \hline
\end{tabular}
\end{table}

\begin{table}[H]
\centering
\caption{ACOPF $\alpha_S{=}0.5$ (0\% $W_3$-feasible, no constrained optimum exists, 3 seeds). }
\label{tab:acopf_05}
\begin{tabular}{lcccc}
\hline
Method & Cost (\$/hr) & Equality (p.u.) & Tier-1 (p.u.) & Tier-2 (p.u.) \\ \hline
DC3 & 71.66 & $0.1523 \pm 0.001$ & $0.0220 \pm 0.0003$ & $1.4169 \pm 0.01$ \\
DC3+recomp & 72.41 & $1.2\text{e-}07 \pm 5\text{e-}09$ & \underline{$0.0132 \pm 0.0002$} & $\mathbf{1.4754 \pm 0.01}$ \\
MResOpt & 73.78 & $3.4\text{e-}07 \pm 4\text{e-}08$ & $\mathbf{0.0054 \pm 0.0003}$ & \underline{$1.4928 \pm 0.01$} \\ \hline
\end{tabular}
\end{table}

Across all three congestion regimes, DC3 drifts off the equality manifold under active correction: its equality violation scales from $0.010$ at $\alpha_S{=}1.0$ to $0.050$ at $0.7$ to $0.152$ at $0.5$.
Here, DC3 has no usable active-correction regime on nonlinear ACOPF (see \Cref{sec:appendix-dc3-analysis} for a full step-size sweep).
Both DC3+recomp and MResOpt preserve the equality manifold at all congestion levels (Eq $\sim 10^{-7}$ to $10^{-10}$), confirming that re-completion is essential for nonlinear constraints.
However, manifold preservation alone is not sufficient: MResOpt achieves $2$--$7\times$ better Tier-1 than DC3+recomp at every $\alpha_S$, showing that the staged architecture enables a learned division of labor that reprojection alone cannot replicate.
At $\alpha_S{=}1.0$ and $0.7$, MResOpt also achieves the best Tier-2; at $\alpha_S{=}0.5$ (fully infeasible), DC3+recomp has slightly better Tier-2 but at substantially worse Tier-1 --- consistent with MResOpt's lexicographic prioritization of high-priority constraints.
The Tier-1 / Tier-2 operating point can be tuned via the stage penalty weights $\lambda_{i,j}$ in the loss (Equation 10), giving practitioners explicit control over this trade-off.
The cost of this priority is moderate --- MResOpt pays $+1.2\%$ OG at $\alpha_S{=}1.0$ relative to DC3 for substantially better constraint satisfaction.

For runtime at $\alpha_S{=}1.0$, DC3 is the fastest at $0.83$ ms/sample (1 Newton PF call), but it cannot maintain the equality manifold under active correction.
MResOpt requires $2.99$ ms/sample (2 Newton PF calls: one initial completion and one re-completion between stages), while DC3+recomp requires $3.74$ ms/sample (10 Newton PF calls).
The dominant cost is the Newton PF solver; MResOpt achieves better Tier-1 than DC3+recomp while being faster, because its staged architecture requires far fewer PF calls.

\section{Conclusion}
\label{sec:conclusion}

We introduced MResOpt, a staged residual architecture that decomposes constraints by priority within the predict–complete–correct paradigm. Across synthetic benchmarks, MResOpt improves Tier-1 feasibility over DC3: detach works in convex settings, while non-detach is needed under nonlinear coupling. On IEEE $30$-bus ACOPF, DC3 has no usable active-correction regime—weak correction is inert, strong correction leaves the equality manifold—whereas MResOpt preserves the manifold and achieves $2$--$7\times$ better Tier-1 feasibility than DC3+recomp, a stronger reprojection baseline; these gains extend to IEEE $57$-bus (\Cref{sec:appendix-57bus}). MResOpt incurs a moderate optimality gap ($+1.2\%$ at $\alpha_S{=}1.0$), and detach degrades under tightly coupled nonlinear constraints such as nonconvex QCQP. Its main advantage is in the practically relevant overconstrained regime ($\alpha_S{=}0.5$, $W_3=\emptyset$), common in congested grids, where feasible-set methods cannot operate but MResOpt remains on the equality manifold while minimizing violations. The schema applies broadly to predict–complete–correct pipelines with separable completion and correction and natural constraint hierarchies, but offers no advantage on smooth landscapes such as DCOPF or ACOPF without line flows, sharpening the claim that it addresses basin trapping from nonlinear constraint coupling.

\paragraph{Limitations.}
MResOpt relies on domain-informed constraint ordering rather than learning it automatically. Its infinite-width GP theory mainly supports the detached variant, while the stronger practical variant is non-detached.
It also incurs overhead from re-completion ($\sim$3.6$\times$ DC3 inference time), though remains faster than DC3+recomp.
Finally, our evaluation is limited to power systems; the generality claim to other constrained optimization domains with natural hierarchies remains to be validated.

\bibliographystyle{unsrt}
\bibliography{references}

@inproceedings{donti2021dc3,
  title     = {{DC3}: A learning method for optimization with hard constraints},
  author    = {Donti, Priya and Rolnick, David and Kolter, J. Zico},
  booktitle = {International Conference on Learning Representations (ICLR)},
  year      = {2021}
}

@inproceedings{nguyen2025fsnet,
  title     = {{FSNet}: Feasibility-Seeking Neural Network for Constrained Optimization with Guarantees},
  author    = {Nguyen, Hoang T. and Donti, Priya L.},
  booktitle = {Advances in Neural Information Processing Systems (NeurIPS)},
  year      = {2025}
}

@article{liang2024hproj,
  title     = {Homeomorphic Projection to Ensure Neural-Network Solution Feasibility for Constrained Optimization},
  author    = {Liang, Enming and Chen, Minghua and Low, Steven H.},
  journal   = {Journal of Machine Learning Research},
  volume    = {25},
  number    = {329},
  pages     = {1--55},
  year      = {2024}
}

@inproceedings{chen2024qcqpnet,
  title     = {{QCQP-Net}: Reliably Learning Feasible Alternating Current Optimal Power Flow Solutions under Constraints},
  author    = {Zeng, Sihan and Kim, Youngdae and Ren, Yuxuan and Kim, Kibaek},
  booktitle = {Proceedings of the 6th Annual Learning for Dynamics \& Control Conference (L4DC)},
  year      = {2024}
}

@article{wu2024deepopfngt,
  title     = {Unsupervised Learning for Solving {AC} Optimal Power Flows: Design, Analysis, and Experiment},
  author    = {Huang, Wanjun and Chen, Minghua and Low, Steven H.},
  journal   = {IEEE Transactions on Power Systems},
  volume    = {39},
  number    = {6},
  pages     = {7102--7114},
  year      = {2024}
}

@inproceedings{amos2017optnet,
  title     = {{OptNet}: Differentiable Optimization as a Layer in Neural Networks},
  author    = {Amos, Brandon and Kolter, J. Zico},
  booktitle = {International Conference on Machine Learning (ICML)},
  pages     = {136--145},
  year      = {2017}
}

@inproceedings{agrawal2019differentiable,
  title     = {Differentiable convex optimization layers},
  author    = {Agrawal, Akshay and Amos, Brandon and Barratt, Shane and Boyd, Stephen and Diamond, Steven and Kolter, J. Zico},
  booktitle = {Advances in Neural Information Processing Systems (NeurIPS)},
  year      = {2019}
}

@article{falck2022multi,
  title={A multi-resolution framework for U-Nets with applications to hierarchical VAEs},
  author={Falck, Fabian and Williams, Christopher and Danks, Dominic and Deligiannidis, George and Yau, Christopher and Holmes, Chris C and Doucet, Arnaud and Willetts, Matthew},
  journal={Advances in Neural Information Processing Systems},
  volume={35},
  pages={15529--15544},
  year={2022}
}

@inproceedings{williams2023unified,
  title     = {A Unified Framework for {U-Net} Design and Analysis},
  author    = {Williams, Christopher and Falck, Fabian and Deligiannidis, George and Holmes, Chris and Doucet, Arnaud and Syed, Saifuddin},
  booktitle = {Advances in Neural Information Processing Systems},
  volume    = {36},
  year      = {2023}
}

@inproceedings{ronneberger2015unet,
  title={U-Net: Convolutional Networks for Biomedical Image Segmentation},
  author={Ronneberger, Olaf and Fischer, Philipp and Brox, Thomas},
  booktitle={Medical Image Computing and Computer-Assisted Intervention--MICCAI 2015: 18th International Conference, Munich, Germany, October 5-9, 2015, Proceedings, Part III 18},
  pages={234--241},
  year={2015},
  organization={Springer}
}

@inproceedings{lee2018deep,
  title={Deep Neural Networks as Gaussian Processes},
  author={Lee, Jaehoon and Bahri, Yasaman and Novak, Roman and Schoenholz, Samuel S and Pennington, Jeffrey and Sohl-Dickstein, Jascha},
  booktitle={International Conference on Learning Representations},
  year={2018}
}

@inproceedings{matthews2018gaussian,
  title={Gaussian process behaviour in wide deep neural networks},
  author={Matthews, Alexander G de G and Rowland, Mark and Hron, Jiri and Turner, Richard E and Ghahramani, Zoubin},
  booktitle={International Conference on Learning Representations (ICLR)},
  year={2018}
}

@inproceedings{chizat2019lazy,
  title={On Lazy Training in Differentiable Programming},
  author={Chizat, L{\'e}na{\"i}c and Oyallon, Edouard and Bach, Francis},
  booktitle={Advances in Neural Information Processing Systems (NeurIPS)},
  volume={32},
  pages={2937--2947},
  year={2019}
}

@inproceedings{he2016deep,
  title={Deep Residual Learning for Image Recognition},
  author={He, Kaiming and Zhang, Xiangyu and Ren, Shaoqing and Sun, Jian},
  booktitle={Proceedings of the IEEE Conference on Computer Vision and Pattern Recognition (CVPR)},
  pages={770--778},
  year={2016}
}

@inproceedings{chen2017photographic,
  title={Photographic Image Synthesis with Cascaded Refinement Networks},
  author={Chen, Qifeng and Koltun, Vladlen},
  booktitle={Proceedings of the IEEE International Conference on Computer Vision (ICCV)},
  pages={1520--1529},
  year={2017}
}

@InProceedings{Han2023Hierarchical,
author="Han, Jihun
and Lee, Yoonsang",
editor="Miky{\v{s}}ka, Ji{\v{r}}{\'i}
and de Mulatier, Cl{\'e}lia
and Paszynski, Maciej
and Krzhizhanovskaya, Valeria V.
and Dongarra, Jack J.
and Sloot, Peter M.A.",
title="Hierarchical Learning to Solve PDEs Using Physics-Informed Neural Networks",
booktitle="Computational Science -- ICCS 2023",
year="2023",
publisher="Springer Nature Switzerland",
address="Cham",
pages="548--562",
isbn="978-3-031-36024-4"
}

@article{franco2023mesh,
  title={Mesh-informed neural networks for operator learning in finite element spaces},
  author={Franco, Nicola Rares and Manzoni, Andrea and Zunino, Paolo},
  journal={Journal of Scientific Computing},
  volume={97},
  number={2},
  pages={35},
  year={2023},
  publisher={Springer}
}

@article{DIMOLA202658,
title = {Numerical solution of mixed-dimensional PDEs using a neural preconditioner},
journal = {Computers \& Mathematics with Applications},
volume = {206},
pages = {58-79},
year = {2026},
issn = {0898-1221},
doi = {https://doi.org/10.1016/j.camwa.2025.12.013},
url = {https://www.sciencedirect.com/science/article/pii/S0898122125005255},
author = {Nunzio Dimola and Nicola {Rares Franco} and Paolo Zunino}
}

@article{klamkin2025pglearn,
  title   = {{PGLearn} -- An Open-Source Learning Toolkit for Optimal Power Flow},
  author  = {Klamkin, Michael and Tanneau, Mathieu and Van Hentenryck, Pascal},
  journal = {arXiv preprint arXiv:2505.22825},
  year    = {2025}
}

@article{babaeinejadsarookolaee2019pglib,
  title   = {The Power Grid Library for Benchmarking {AC} Optimal Power Flow Algorithms},
  author  = {Babaeinejadsarookolaee, Sogol and Birchfield, Adam and Christie, Richard D. and Coffrin, Carleton and DeMarco, Christopher and Diao, Ruisheng and Ferris, Michael and Fliscounakis, Stephane and Greene, Scott and Huang, Renke and Josz, Cedric and Korab, Roman and Lesieutre, Bernard and Maeght, Jean and Mak, Terrence W. K. and Molzahn, Daniel K. and Overbye, Thomas J. and Panciatici, Patrick and Park, Byungkwon and Snodgrass, Jonathan and Tbaileh, Ahmad and Van Hentenryck, Pascal and Zimmerman, Ray},
  journal = {arXiv preprint arXiv:1908.02788},
  year    = {2019}
}

\newpage
\appendix
\section{Appendix}

\subsection{Notation}
% We will need to distinguish two related notions when reasoning about feasibility. 
% The \emph{cumulative} feasible set $W_i$ defined above collects all constraints \emph{up to} priority $i$; by construction $W_i \subseteq W_{i-1}$. 
% In contrast, when discussing whether a single lower-priority constraint is incompatible with the higher-priority cumulative set, we will use the \emph{standalone} set $S_i \coloneqq \{w \in \mathbb{W} \mid C_i(w) = 0\}$ for the $i$-th constraint alone.
% Infeasibility statements such as ``the high-priority feasible set cannot be jointly satisfied with the next constraint'' are then expressed as $W_{i-1} \cap S_i = \emptyset$, not $W_{i-1} \cap W_i = \emptyset$ (which would be vacuous since $W_i \subseteq W_{i-1}$). For ACOPF (\Cref{assumption:lex-order}), the cumulative sets specialize to a base equality manifold $W_1 = \{C_{\mathrm{eq}} = 0\}$, then $W_2 = W_1 \cap \{C_{\mathrm{box}} \leq 0\}$, then $W_3 = W_2 \cap \{C_{\mathrm{flow}} \leq 0\}$; the standalone box-bound and line-flow sets are reported in the experiments as \emph{Tier-1} ($C_{\mathrm{box}}$) and \emph{Tier-2} ($C_{\mathrm{flow}}$) violations. We use \emph{Stage $k$} for the architectural block (completion + correction) responsible for advancing the iterate from $W_k$ toward $W_{k+1}$, so Stage~1 enforces equality + Tier-1 (box) and Stage~2 additionally enforces Tier-2 (line flows). Equality is reported as a separate column in every results table; Tier-$k$ refers exclusively to the inequality group at priority $k$.

Throughout this work, we adopt a unified notation to formalize the subspace preconditioning framework and its theoretical analysis.
The physical system state is represented by an input vector $v \in \mathbb{V} = \mathbb{R}^n$ , which is mapped to a proposed optimal solution $w = N_{\theta}(v)$ within the solution space $\mathbb{W} = \mathbb{R}^d$. 
The neural network $N_{\theta}$ produces stage-wise latent proposals $z_i \in \mathbb{R}^D$ with reduced dimensionality $D < d$, which are transformed into completed variables $\tilde{w}_i$ via level-specific completion operators $\mathcal{T}_i$ and subsequently into corrected variables $\hat{w}_i$ through projection operators $\mathcal{P}_i$ that enforce the $i$-th manifold $W_i$. 
These manifolds form a descending filtration $W_r \subset W_{r-1} \subset \dots \subset W_1 \subset W_0 = \mathbb{W}$, where each subspace $W_i = \{w \in \mathbb{W} \mid C_j(w) = 0 \text{ for all } j \le i\}$ represents the joint null set of ordered constraint components $\{C_j\}_{j=1}^r$ ranked by lexicographical priority. 
For $i > 1$, the recursive architecture generates proposals through the conditional residual step $z_i = \text{Detach}(z_{i-1}) + N_i(\text{Detach}(z_{i-1}) \mid v)$, where the stop-gradient operation $\text{Detach}(\cdot)$ enforces a greedy priority by decoupling stage-wise gradients to ensure that operational constraints do not compromise the physical consistency of foundational grid laws established in $W_1$. 
To characterize the infinite-width limit of these networks, let $w_l$ denote the width of layer $l$ and $\theta \in \mathbb{R}^p$ denote the collective parameter vector. 
The pre-activations and activations at each layer converge in distribution to a Gaussian Process governed by the Neural Network Gaussian Process kernel $K$ at initialization. 
Training dynamics are described by the Neural Tangent Kernel $\Theta$, which remains approximately constant during the lazy training regime. 
The evolution of the output $w_k(t)$ follows a linear differential equation governed by the local kernel such that $\frac{dw_k(t)}{dt} = \eta\Theta_k(\mathcal{W} - w_k(t))$. 
For graph-structured systems such as the electric grid, these domains are extended to graphs $G = (V, E)$, where the Graph Neural Tangent Kernel correspondence is utilized to describe structural alignment and message-passing dynamics.

\subsection{ACOPF Problem Formulation}
\label{sec:acopf-formulation}

The AC Optimal Power Flow (ACOPF) problem minimizes generation cost subject to nonlinear power flow physics and operational limits.
For a network with $n_b$ buses, $n_g$ generators, and $n_\ell$ branches:
\begin{align}
\min_{P_g, Q_g, V_m, V_a} \quad & \sum_{i=1}^{n_g} \left( c_{2,i} P_{g,i}^2 + c_{1,i} P_{g,i} \right) \label{eq:acopf-obj} \\
\text{s.t.} \quad & P_{g,i} - P_{d,i} = \sum_{k=1}^{n_b} V_{m,i} V_{m,k} ( G_{ik} \cos \theta_{ik} + B_{ik} \sin \theta_{ik} ), \label{eq:pbal} \\
& Q_{g,i} - Q_{d,i} = \sum_{k=1}^{n_b} V_{m,i} V_{m,k} ( G_{ik} \sin \theta_{ik} - B_{ik} \cos \theta_{ik} ), \label{eq:qbal} \\
& P_{g,i}^{\min} \leq P_{g,i} \leq P_{g,i}^{\max}, \quad Q_{g,i}^{\min} \leq Q_{g,i} \leq Q_{g,i}^{\max}, \label{eq:gen-bounds} \\
& V_{m,i}^{\min} \leq V_{m,i} \leq V_{m,i}^{\max}, \label{eq:vm-bounds} \\
& |S_\ell^f|^2, |S_\ell^t|^2 \leq (\bar{S}_\ell)^2, \quad \forall \ell \in \{1, \dots, n_\ell\}, \label{eq:flow-bounds}
\end{align}
where $\theta_{ik} = V_{a,i} - V_{a,k}$, $G + jB = Y_{\text{bus}}$ is the network admittance matrix, and $S_\ell^f, S_\ell^t$ are the complex power flows at the from/to ends of branch $\ell$.

The full solution vector is $w = [P_g, Q_g, V_m, V_a] \in \mathbb{R}^d$ with $d = 2 n_g + 2 n_b$.
Following DC3 \cite{donti2021dc3}, the neural network predicts only the \emph{partial} (independent) variables $z = [P_{g}^{\text{PV}}, V_m^{\text{slack} \cup \text{PV}}] \in \mathbb{R}^D$, and a Newton power flow solver $\mathcal{T}$ completes the remaining (dependent) variables to satisfy \eqref{eq:pbal}--\eqref{eq:qbal}.

The inequality constraints map to the tier hierarchy in \Cref{assumption:lex-order}:
\begin{itemize}
    \item \textbf{Tier-1} ($C_{\text{box}}$): generator limits \eqref{eq:gen-bounds} and voltage limits \eqref{eq:vm-bounds},
    \item \textbf{Tier-2} ($C_{\text{flow}}$): thermal limits \eqref{eq:flow-bounds}.
\end{itemize}

We follow DC3's equality-completion parameterization and generator/voltage box constraints \cite{donti2021dc3,nguyen2025fsnet}, and extend the ACOPF setting with explicit thermal line-flow constraints (Tier-2); for a more detailed presentation (e.g., separating real and reactive line flows explicitly) we refer the reader to \cite{donti2021dc3}.

\subsection{Subspace Preconditioning Examples}

\begin{example}[Bimodal Basin Trapping]
\label{ex:basin-trap}
Consider finding a feasible point for $f(w) = \| w - (4,4) \|_2^2$ in $\mathbb{R}^2$, initialized at $(0,0)$. 
The system is governed by a physical constraint $C_1(w) = \max(0, (w_1-3)^2 + (w_2-3)^2 - 3)$ and an operational constraint $C_2(w) = \max(0, (w_1-3)^2 + ((w_2-3)^2 - 1.1)^2 - 1.01)$. 
Simultaneous enforcement of $\{C_1, C_2\}$, or prioritizing $C_2$, causes the zero-centered initialization to descend into a suboptimal local basin near $(3, 2.1)$. 
Because the joint manifold $S_1 \cap S_2$ (the intersection of the two standalone constraint sets) is non-convex and disconnected, the solver becomes trapped.
Conversely, as shown in \Cref{fig:subspace preconditioning_comparison}, adopting the filtration $W_1 \subset W_0$ first relaxes the search to the circle defined by $C_1$. 
This acts as a global preconditioner, anchoring the solution in the upper quadrant. 
Once $C_2$ is introduced, the model is already positioned within the correct basin to reach the global optimum at $(4,4)$.

We validate the architectural inductive bias using a synthetic benchmark to identify a global optimum $w^* = (4, 4)$ subject to a physical manifold $W_{circle}$ and operational manifold $W_{ovals}$. 
The recursive ResNet is compared against a Multi-Layer Perceptron baseline using a simultaneous weighted penalty loss $\mathcal{L} = f(w) + \lambda_1 \|C_{1}\|^2 + \lambda_2 \|C_{2}\|^2$. 
While the simultaneous baseline consistently descends into the nearest local basin, reaching a suboptimal point, the subspace preconditioning model bypasses this trapping region by anchoring the initialization $N_1$ to the physical manifold $W_1$.  
The subsequent residual stage $N_i$ successfully navigates from this physically consistent quadrant toward the global optimum. 
This gravity well scenario serves as a stress test for distributional collapse; as shown in Table \ref{tab:basin_trapping}, the simultaneous method exhibits a 100\% trapping rate, whereas the proposed framework achieves a 100\% success rate across 500 test samples. 
By isolating $L_1$ in Stage 1, the network establishes a global anchor $z_1$ that prevents the optimizer from prioritizing the immediate reduction of high-weight operational penalties $L_2$ at the cost of the global path. 
These results confirm that the recursive, detached architecture acts as a structural preconditioner, ensuring that prioritizing fundamental physics provides an optimal path for secondary operational refinements.

\begin{table}[H]
\centering
\caption{Comparative Performance on the Bimodal Basin Trapping Benchmark ($N=500$ test samples).}
\label{tab:basin_trapping}
\begin{tabular}{lccc}
\hline
Method & Average Objective Cost & Physical Violation ($L_1$) & Success Rate (\%) \\ \hline
Simultaneous Baseline & 3.4841 & 0.0000 & 0.0\%  \\
MResOpt & 0.0320 & 0.0000 & 100.0\%  \\ \hline
\end{tabular}
\end{table}
\end{example}

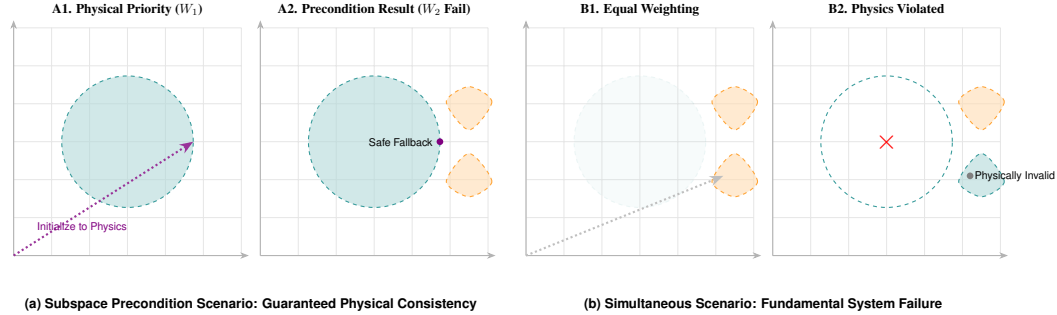
\begin{figure}[h!]
    \centering
    \resizebox{\textwidth}{!}{
        \begin{tikzpicture}[scale=1.0, >=Stealth]

            % --- Global Styles ---
            \tikzset{
                grid lines/.style={color=gray!20, thin},
                axis/.style={->, thick, color=gray!60},
                h1 region/.style={fill=niceTeal!30, draw=niceTeal!80, dashed, thick, fill opacity=0.6},
                h2 region/.style={fill=niceOrange!30, draw=niceOrange!80, dashed, thick, fill opacity=0.6},
                target pt/.style={circle, fill=nicePurple, inner sep=2.5pt},
                good arrow/.style={->, ultra thick, color=nicePurple!80, dotted},
                bad arrow/.style={->, ultra thick, color=gray!50, dotted},
                lbl/.style={font=\footnotesize\sffamily},
                sub caption/.style={font=\bfseries\sffamily, anchor=north, yshift=-1.0cm}
            }

            \newcommand{\drawBaseGrid}{
                \draw[grid lines] (0,0) grid (6,6);
                \draw[axis] (0,0) -- (6.2,0);
                \draw[axis] (0,0) -- (0,6.2);
            }

            % ==========================================
            % SCENARIO A: Subspace Precondition (Safe Physical Fallback)
            % ==========================================
            \begin{scope}[xshift=0cm]
                \node[anchor=south, font=\bfseries] at (3,6.2) {A1. Physical Priority ($W_1$)};
                \drawBaseGrid
                \filldraw[h1 region] (3,3) circle ({sqrt(3)});
                \draw[good arrow] (0,0) -- (4.732, 3);
                \node[lbl, nicePurple, anchor=south west] at (0.5,0.5) {Initialize to Physics};
            \end{scope}

            \begin{scope}[xshift=6.5cm]
                \node[anchor=south, font=\bfseries] at (3,6.2) {A2. Precondition Result ($W_2$ Fail)};
                \drawBaseGrid
                \fill[h1 region] (3,3) circle ({sqrt(3)});
                % Ovals shifted right to demonstrate disjoint sets
                \filldraw[h2 region, smooth cycle] plot coordinates {(5.5, 4.45) (6.1, 4.05) (5.5, 3.31) (4.9, 4.05)};
                \filldraw[h2 region, smooth cycle] plot coordinates {(5.5, 2.69) (6.1, 1.95) (5.5, 1.55) (4.9, 1.95)};
                
                \fill[nicePurple] (4.732, 3) circle (2.5pt) node[left, lbl, black, xshift=-2pt] {Safe Fallback};
            \end{scope}

            % ==========================================
            % SCENARIO B: SIMULTANEOUS (System Violation)
            % ==========================================
            \begin{scope}[xshift=13.5cm]
                \node[anchor=south, font=\bfseries] at (3,6.2) {B1. Equal Weighting};
                \drawBaseGrid
                \fill[h1 region, opacity=0.1] (3,3) circle ({sqrt(3)});
                \filldraw[h2 region, smooth cycle] plot coordinates {(5.5, 4.45) (6.1, 4.05) (5.5, 3.31) (4.9, 4.05)};
                \filldraw[h2 region, smooth cycle] plot coordinates {(5.5, 2.69) (6.1, 1.95) (5.5, 1.55) (4.9, 1.95)};
                \draw[bad arrow] (0,0) -- (5.2, 2.1);
            \end{scope}

            \begin{scope}[xshift=20cm]
                \node[anchor=south, font=\bfseries] at (3,6.2) {B2. Physics Violated};
                \drawBaseGrid
                \draw[niceTeal!80, dashed, thick] (3,3) circle ({sqrt(3)});
                \filldraw[h2 region, smooth cycle] plot coordinates {(5.5, 4.45) (6.1, 4.05) (5.5, 3.31) (4.9, 4.05)};
                \filldraw[h1 region, smooth cycle] plot coordinates {(5.5, 2.69) (6.1, 1.95) (5.5, 1.55) (4.9, 1.95)};
                
                \fill[gray] (5.2, 2.1) circle (2.5pt) node[right, lbl, black] {Physically Invalid};
                \node[red, font=\huge] at (3,3) {$\times$};
            \end{scope}

            % Horizontal labeling under each scenario
            \node[sub caption] at (6.25, 0) {(a) Subspace Precondition Scenario: Guaranteed Physical Consistency};
            \node[sub caption] at (19.75, 0) {(b) Simultaneous Scenario: Fundamental System Failure};

        \end{tikzpicture}
    }
    \caption{
    A comparison of subspace preconditioning and Simultaneous satisfaction under over-constrained conditions. 
    In Scenario A, subspace preconditioning ensures the model defaults to a physically valid solution. 
    In Scenario B, simultaneous satisfaction allows operational goals to pull the solution out of the physical manifold, resulting in an invalid system state.}
    \label{fig:subspace preconditioning_infeasibility}
\end{figure}

\begin{example}[Basin Trapping and Physical Fallback]
\label{ex:physical-fallback}
Consider finding a feasible point for $f(w) = \| w - (4,4) \|_2^2$ in $\mathbb{R}^2$, initialized at $(0,0)$. 
The system is governed by a physical constraint $C_1(w) = \max(0, (w_1-3)^2 + (w_2-3)^2 - 3)$ and an operational constraint $C_2(w) = \max(0, (w_1-3)^2 + ((w_2-3)^2 - 1.1)^2 - 1.01)$. 
As illustrated in \Cref{fig:subspace preconditioning_infeasibility}, subspace preconditioning successfully identifies a physically feasible solution even when no global feasible solution exists. 
In scenarios where operational requirements are incompatible with physics ($W_1 \cap S_2 = \emptyset$, where $S_2$ is the standalone operational constraint set), the model defaults to a \emph{safe fallback} on the physical manifold because it prioritizes the satisfaction of $W_1$.
Conversely, a simultaneous enforcement approach often yields a solution that satisfies neither constraint, violating fundamental physical laws in a failed attempt to reach the operational basin.

To test the theoretical guarantees regarding disjoint sets, we simulate a scenario where $W_1 \cap S_2 = \emptyset$ (the standalone operational constraint $S_2$ is incompatible with the high-priority cumulative set $W_1$).
In accordance with \Cref{lemma:infeasible-cost}, the weighted penalty baseline produces solutions between the sets, violating both physics and operational goals.  
Conversely, the subspace preconditioning architecture demonstrates a Safe Fallback property; because $N_1$ is optimized greedily and the subsequent gradient is detached, the network maintains near-zero violation of $W_1$ even when $N_2$ fails. 
As shown in Table \ref{tab:results}, while Stage 1 achieves perfect physical satisfaction, Stage 2 refinement exhibits drift when lexicographic priority ($\lambda_1 \gg \lambda_2$) is absent, pulling the solution into infeasible space. This empirical result validates the necessity of a system that can revert to the Stage 1 output $z_1$ in over-constrained grid states to ensure structural safety.

\begin{table}[H]
\centering
\caption{Comparison of Mean and Maximum Constraint Violations across 500 Samples.}
\label{tab:results}
\begin{tabular}{llcc}
\hline
Method / Stage & Metric & Average Error & Max Error \\ \hline
Simultaneous (Baseline) & $L_1$ (Physics) & 6.99906 & 7.38585 \\
 & $L_2$ (Operational) & 7.00127 & 7.15303 \\ \hline
MResOpt Stage 1 & $L_1$ (Physics) & 0.00000 & 0.00000 \\ \hline
MResOpt Stage 2 & $L_1$ (Physics) & 7.00254 & 7.29045 \\
 & $L_2$ (Operational) & 6.99857 & 7.18668 \\ \hline
\end{tabular}
\end{table}
\end{example}

\subsection{Delayed Mathematics and Proofs}
\label{app:proofs}

\begin{assumption}[Lexicographic Constraint Ordering]\label{assumption:lex-order}
For Economic Dispatch, we identify three constraint groups:
\begin{itemize}
\item $C_{\text{eq}}$: AC power-balance equations (equality manifold),
\item $C_{\text{box}}$: Generator and voltage box bounds (local safety),
\item $C_{\text{flow}}$: Line flow thermal limits (network security).
\end{itemize}

These induce the nested feasible sets ($W_1 \supseteq W_2 \supseteq W_3$):
\begin{itemize}
\item $W_1 = \{w : C_{\text{eq}}(w) = 0\}$ --- equality manifold,
\item $W_2 = \{w \in W_1 : C_{\text{box}}(w) \leq 0\}$ --- equality + box bounds,
\item $W_3 = \{w \in W_2 : C_{\text{flow}}(w) \leq 0\}$ --- full feasibility.
\end{itemize}

The equality constraints $C_{\text{eq}}$ are enforced exactly via the completion operator $\mathcal{T}$ (Newton power flow), not by the correction stages.
The inequality constraints are partitioned into two tiers for the correction operator:
\textbf{Tier-1} ($C_{\text{box}}$) and \textbf{Tier-2} ($C_{\text{flow}}$).
Stage~1 targets $W_2$ (equality + box bounds); Stage~2 refines toward $W_3$ (adding line flows).

In the tables, we report \textbf{Equality}, \textbf{Tier-1} (box violations), and \textbf{Tier-2} (flow violations) separately.

% \begin{enumerate}
%       \item $C_1$ represents Power Generation limits

%     \item $C_1$ represents fundamental physical laws, specifically the \emph{Power Balance Equations}:
%     \begin{equation*}
%         C_1(P, V, \theta) = \left( P_i - \sum_{j \in \mathcal{N}_i} |V_i||V_j|(G_{ij}\cos\theta_{ij} + B_{ij}\sin\theta_{ij}) \right)^2 = 0
%     \end{equation*}
%     These equalities define the manifold on which the system state must exist.
    
%     \item $C_2$ represents critical state bounds, specifically \emph{Voltage Magnitude Limits}:
%     \begin{equation*}
%         C_2(V) = \max(0, V_i^{\min} - |V_i|) + \max(0, |V_i| - V_i^{\max}) = 0
%     \end{equation*}
%     Maintaining voltage stability is the primary operational safety requirement.

%     \item $C_3$ represents \emph{Thermal Line Limits} (Transmission Capacity):
%     \begin{equation*}
%         C_3(S_{ij}) = \max(0, |S_{ij}| - S_{ij}^{\max}) = 0, \quad \forall (i,j) \in \mathcal{E}
%     \end{equation*}
%     These ensure transmission lines do not overheat, appearing before global security checks.
    
%     % \item $C_4$ represents \emph{Security Constraints}, such as $N-1$ contingency stability:
%     % \begin{equation*}
%     %    C_4(x, u) = \max(0, h_k(x, u)) = 0, \quad \forall k \in \mathcal{K}_{\text{contingencies}}
%     % \end{equation*}
%     % These ensure robustness against component failure.
% \end{enumerate}
\end{assumption}

Let $\sigma_i$ be the $\sigma$-algebra with respect to which the constraint functions $\bigcup_{j<i} \{ C_j \}$ are measurable. 
Let the system state estimate at stage $i$, denoted $w_i$, be a $\sigma_i$-measurable function. 
Corresponding to the filtration of approximation spaces, we define the loss function $\mathcal{L}_i$ for the $i$-th stage of the network. 
This objective aggregates the economic cost $f$ with the cumulative constraint violations up to level $i$:

\begin{align}
    \mathcal{L}_i(w_i) = \gamma f(w_i) + \sum_{j=1}^i \lambda_{i,j} C_j(w_i)
\end{align}

where $\gamma \geq 0$ weights the economic objective, and $\lambda_{i,j} > 0$ are penalty coefficients for the constraints. 
This formulation ensures that at each stage $i$, the model seeks a solution $w_i$ that is not only feasible with respect to the partial constraint set $\{C_1, \dots, C_i\}$ but also economically efficient. 
Similar to approaches like DC3 \cite{donti2021dc3}, by incorporating $f$ throughout the hierarchy, we ensure that the search trajectory is consistently guided toward the optimal region of the manifold, preventing the selection of feasible but arbitrarily expensive intermediate solutions.

% Let $\sigma_i$ be the $\sigma$-algebra which makes each of the functions in the set $\cup_{j<i} \{ C_j \}$ measurable. 
% Let $w_i|v$ be a $\sigma_i$ measurable function. 
% Corresponding to the filtration of approximation spaces, we define the loss function $\mathcal{L}_i$ for the $i$-th stage of the network. 
% This objective aggregates the economic cost $f$ with the cumulative constraint violations up to level $i$ using the $w_i$ parameterization:

% \begin{align}
%     \mathcal{L}_i(w_i) = \gamma f(w_i) + \sum_{j=1}^i \lambda_{i,j} C_j(w_i)
% \end{align}

% where $y$ is the network output (system state), $\gamma \geq 0$ weights the economic objective, and $\lambda_{i,j} > 0$ are penalty coefficients for the constraints. 
% This formulation ensures that at each stage $i$, the model seeks a solution that is not only feasible with respect to the partial constraint set $\{C_1, \dots, C_i\}$ but also economically efficient. 
% Like DC3, by including $f$ throughout the hierarchy, we ensure that the search trajectory is consistently guided toward the optimal region of the manifold, preventing the selection of feasible but arbitrarily expensive intermediate solutions.

\begin{lemma}[Infeasibility of Weighted Penalties for Disjoint Sets]\label{lemma:infeasible-cost}
    Let $A_1,A_2 \subset \mathbb{W}$ be two disjoint, non-empty closed sets. Let $d(w, A_i) = \inf_{a \in A_i} \|w - a\|_2$ denote the distance function to each set.
    For any convex combination of the penalty terms $\mathcal{L}(w) = \alpha d(w, A_1)^2 + (1-\alpha) d(w,A_2)^2$ with $\alpha \in (0,1)$, the minimizer $w^* = \arg\min_w \mathcal{L}(w)$ satisfies $w^* \notin A_1 \cup A_2$.
\end{lemma}

\begin{proof}
    The gradient of the squared distance function to a closed set is $\nabla d(w, A_i)^2 = 2(w - \text{Proj}_{A_i}(w))$.
    The first-order optimality condition for $\mathcal{L}(w)$ implies:
    \begin{align}
        \alpha(w^* - \text{Proj}_{A_1}(w^*)) + (1-\alpha)(w^* - \text{Proj}_{A_2}(w^*)) = 0.
    \end{align}
    Rearranging shows that $w^*$ is a convex combination of its projections onto the two sets.
    Since the sets are disjoint, the distance between the projections is strictly positive, forcing $w^*$ to lie strictly between the sets, thus violating both constraints.
\end{proof}

\begin{corollary}\label{cor:physical-first}
The sequence of feasible sets $\{W_i\}_{i=1}^r$ in the subspace preconditioning process is monotonically non-increasing with respect to set inclusion, such that: $W_r \subseteq W_{r-1} \subseteq \dots \subseteq W_1$. 
Consequently, if the process reaches a state $w^{(i)} \in W_i$, it is guaranteed to satisfy all higher-priority constraints $\{W_j\}_{j=1}^{i-1}$.
\end{corollary}

\begin{proof}(\textbf{\Cref{lemma:ResNetGP}})
By the universality of the Gaussian Process limit for wide networks \cite{lee2018deep}, the residual branch converges in distribution $\mathcal{R}_{\text{res}} \xrightarrow{d} \mathcal{G}$ as $n \to \infty$, where $\mathcal{G}$ is a centered Gaussian Process $\mathcal{GP}(0, K)$.
Consider any finite collection of points $\{v_i\}_{i=1}^k \subset \mathbb{V}$. Let $R_{\text{res}}$ and $M$ be the vector evaluations on these points:
\begin{align}
    R_{\text{res}} = (\mathcal{R}_{\text{res}}(v_1), \dots, \mathcal{R}_{\text{res}}(v_k))^\top, \quad M = (\mu(v_1), \dots, \mu(v_k))^\top.
\end{align}
In the limit $n \to \infty$, the random vector $R_{\text{res}}$ follows a multivariate normal distribution $\mathcal{N}(0, \Sigma)$, where the covariance matrix entries are $\Sigma_{ij} = K(v_i, v_j)$.
The output vector of the ResNet is given by the translation $R = M + R_{\text{res}}$. Since the multivariate normal distribution is closed under affine transformations, $R$ is distributed as:
\begin{equation}
    R \sim \mathcal{N}(M + 0, \Sigma) = \mathcal{N}(M, \Sigma).
\end{equation}
As this holds for any finite set of points in $\mathbb{V}$, the process $\mathcal{R}$ is by definition a Gaussian Process with mean function $\mu$ and covariance function $K$.
\end{proof}

\begin{proof}(\textbf{\Cref{lemma:vanishing-influence}})
In the infinite-width limit, the network $\mathcal{R}$ is linearized around its initialization $\mathcal{R}_0$. 
Under gradient flow on the squared loss, the output at a test point $v^*$ after infinite time is given by the closed-form kernel regression expression:
\begin{align}
    \mathcal{R}_{\infty}(v^*) = \mathcal{R}_0(v^*) + \Theta(v^*, \mathcal{V}) \Theta(\mathcal{V}, \mathcal{V})^{-1} (Y - \mathcal{R}_0(\mathcal{V})).
\end{align}
By the assumption of vanishing correlation, the kernel vector converges to zero: $\Theta(v^*, \mathcal{V}) \to 0$. 
Since the training kernel matrix $\Theta(\mathcal{V}, \mathcal{V})$ is positive definite and the residual vector $(Y - \mathcal{R}_0(\mathcal{V}))$ is bounded for any finite training set $\mathcal{V}$, the product term vanishes:
\begin{align}
    \mathcal{R}_{\infty}(v^*) = \mathcal{R}_0(v^*) + o(1).
\end{align}
Substituting the result that the initialization follows the prior $\mathcal{R}_0(v^*) \sim \mathcal{N}(\mu(v^*), K(v^*, v^*))$, we obtain the desired convergence in distribution.
\end{proof}

% \begin{proof}(\textbf{\Cref{lemma:vanishing-influence}})
% In the infinite-width limit, the network $\mathcal{R}$ is linearized around its initialization $\mathcal{R}_0$. 
% Under gradient flow on the squared loss, the output at a test point $x^*$ after infinite time is given by the closed-form kernel regression expression:
% \begin{equation}
%     \mathcal{R}_{\infty}(v^*) = \mathcal{R}_0(v^*) + \Theta(x^*, X) \Theta(X, X)^{-1} (Y - \mathcal{R}_0(X))
% \end{equation}
% By the assumption of vanishing correlation, $\Theta(x^*, X) \to \mathbf{0}^\top$. 
% Since the training kernel matrix $\Theta(X, X)$ is positive definite and the residual $(Y - \mathcal{R}_0(X))$ is bounded for any finite $X$, the product term vanishes:
% \begin{equation}
%     \mathcal{R}_{\infty}(x^*) = \mathcal{R}_0(x^*) + o(1)
% \end{equation}
% Substituting the result where $\mathcal{R}_0(x^*) \sim \mathcal{N}(\mu(x^*), K(x^*, x^*))$, we obtain the desired convergence in distribution.
% \end{proof}

\begin{lemma}[Asymptotic Decay for Bounded Activations]
\label{lemma:asymptotic-decay}
Let $\mathcal{R}$ be an infinitely wide ResNet with a $C^1$ bounded activation function $\sigma$ satisfying $\lim_{|z| \to \infty} \sigma'(z) = 0$. For any fixed training set $X$, the kernel correlation between a test point $x^*$ and the training data vanishes as the test point moves to infinity: $\lim_{\|v^*\| \to \infty} \Theta(v^*, \mathcal{V}) = \mathbf{0}^\top.$
\end{lemma}

\begin{proof}(\textbf{\Cref{lemma:asymptotic-decay}})
Consider the recursive definition of the Neural Tangent Kernel (NTK) at layer $l$:
\begin{equation}
    \Theta^{(l)}(v, v') = \Theta^{(l-1)}(v, v') \dot{K}^{(l)}(v, v') + K^{(l)}(v, v')
\end{equation}
where the derivative kernel $\dot{K}^{(l)}$ is defined by the expectation:
\begin{equation}
    \dot{K}^{(l)}(v, v') = \sigma_w^2 \mathbb{E}_{(z_1, z_2) \sim \mathcal{N}(0, \Sigma^{(l-1)})} [\sigma'(z_1)\sigma'(z_2)]
\end{equation}
As $\|v^*\| \to \infty$, the variance of the pre-activations $K^{(l-1)}(v^*, v^*)$ diverges to infinity. This is because the initial covariance $K^{(0)}(v^*, v^*) = \frac{\sigma_w^2}{d_{in}} \|v^*\|^2 + \sigma_b^2$ is unbounded in $\|v^*\|$, and the ResNet architecture preserves this growth through subsequent layers.
Under the assumption $\lim_{|z| \to \infty} \sigma'(z) = 0$, as the variance $K^{(l-1)}(v^*, v^*)$ tends to infinity, the mass of the Gaussian measure $\mathcal{N}(0, \Sigma^{(l-1)})$ spreads such that it concentrates in regions where the product $\sigma'(z_1)\sigma'(z_2)$ approaches zero. Since $\sigma'$ is bounded, we can apply the Dominated Convergence Theorem, yielding $\dot{K}^{(l)}(v^*, v_i) \to 0$ for all $v_i \in \mathcal{V}$.
Because the gradient of the network $\nabla_\theta \mathcal{R}$ is multilinearly dependent on these vanishing derivative terms $\sigma'$, the inner product $\Theta(v^*, v_i) = \langle \nabla_\theta \mathcal{R}(v^*), \nabla_\theta \mathcal{R}(v_i) \rangle$ vanishes. Consequently, the vector $\Theta(v^*, \mathcal{V})$ approaches the zero vector in the limit:
\begin{equation}
    \lim_{\|v^*\| \to \infty} \Theta(v^*, \mathcal{V}) = \mathbf{0}^\top.
\end{equation}
\end{proof}

% \begin{proof}(\Cref{lemma:asymptotic-decay})
% \color{nicePurple}
% Consider the recursive definition of the NTK at layer $l$:
% \begin{equation}
%     \Theta^{(l)}(x, x') = \Theta^{(l-1)}(x, x') \dot{K}^{(l)}(x, x') + K^{(l)}(x, x')
% \end{equation}
% where the derivative kernel $\dot{K}^{(l)}$ is defined by the expectation:
% \begin{equation}
%     \dot{K}^{(l)}(x, x') = \sigma_w^2 \mathbb{E}_{z \sim \mathcal{N}(0, \Sigma^{(l-1)})} [\sigma'(z_1)\sigma'(z_2)]
% \end{equation}
% As $\|x^*\| \to \infty$, the variance of the pre-activations $K^{(l-1)}(x^*, x^*)$ diverges to infinity. Under the assumption $\lim_{|z| \to \infty} \sigma'(z) = 0$, the mass of the Gaussian measure $\mathcal{N}(0, \Sigma^{(l-1)})$ concentrates in regions where the product $\sigma'(z_1)\sigma'(z_2)$ approaches zero. By the Dominated Convergence Theorem, $\dot{K}^{(l)}(x^*, x_i) \to 0$ for all $x_i \in X$. 

% Since the gradient of the network $\nabla_\theta \mathcal{R}$ is multi-linearly dependent on these vanishing derivative terms $\sigma'$, the inner product $\Theta(x^*, x_i) = \langle \nabla_\theta \mathcal{R}(x^*), \nabla_\theta \mathcal{R}(x_i) \rangle$ vanishes. Consequently, the vector $\Theta(x^*, X)$ approaches the zero vector in the limit.
% \end{proof}

\begin{proof}(\textbf{\Cref{thm:msrl-equiv}})
Let $E_k = \mathcal{W}_k - w_k$ be the residual error at stage $k$, where $Y$ is the target label and $w_k$ is the output of the $k$-th ResNet block. 
Due to gradient detachment, the parameters $\theta_k$ evolve independently within a tangent space anchored by the output of the prior stage $w_{k-1}$. 

By \Cref{lemma:ResNetGP}, each block $\mathcal{R}_k$ initializes as a Gaussian Process $\mathcal{GP}(w_{k-1}, K_k)$. In the infinite-width limit, the evolution of the activation $w_k$ follows the linear ODE:
\begin{equation}
    \dot{w}_k(t) = \eta \Theta_k ( \mathcal{W}_k - w_k(t) ).
\end{equation}
Integrating over training time $t$ yields the staged update:
\begin{equation}
    w_k(t) = w_{k-1} + (I - e^{-\eta \Theta_k t})(\mathcal{W}_k - w_{k-1}) + \mathcal{E}(\delta).
\end{equation}
The perturbation term $\mathcal{E}(\delta)$ stems from non-zero correlation between the current tangent space and the prior manifold. 
Under the \emph{Bounded Interference} assumption, this correlation is suppressed; as the representation $w_{k-1}$ shifts away from the previous distribution, \Cref{lemma:asymptotic-decay} ensures that the cross-kernel terms $\Theta(w_k, w_{k-1}) \to 0$. 
This effectively decouples the optimization of the $k$-th block $\mathcal{R}_k$ from the functional form of the preceding block $\mathcal{R}_{k-1}$. 
As $t \to \infty$, the update converges to the kernel regression solution:
\begin{equation}
    w_k = w_{k-1} + \Theta_k(w_{k-1}, \mathcal{V}) \Theta_k(\mathcal{V}, \mathcal{V})^{-1} (\mathcal{W}_k - w_{k-1}) + O(\delta).
\end{equation}
The error propagates across stages according to $\|E_k\|_2 \le (e^{-\eta \lambda t} + C\delta)\|E_{k-1}\|_2$. In the limit $\delta \to 0$ and $t \to \infty$, this update simplifies to $w_k = w_{k-1} + \text{Proj}_{\Theta_k}(E_{k-1})$, which is identically the form of stage-wise gradient boosting in the function space.
\end{proof}

% \begin{proof} (\Cref{thm:msrl-equiv})
% \color{nicePurple}
% Let $E_k = Y - y_k$ be the residual error at stage $k$. Due to gradient detachment, $\theta_k$ evolves independently in a tangent space anchored by the prior stage. By \Cref{lemma:ResNetGP}, each block $\mathcal{R}_k$ initializes as a $\mathcal{GP}(y_{k-1}, K_k)$. In the infinite-width limit, the evolution of $y_k$ follows the linear ODE:
% \begin{equation}
%     \dot{y}_k(t) = \eta \Theta_k ( Y - y_k(t) ).
% \end{equation}
% Integrating over training time $t$ yields the update:
% \begin{equation}
%     y_k(t) = y_{k-1} + (I - e^{-\eta \Theta_k t})(Y - y_{k-1}) + \mathcal{E}(\delta).
% \end{equation}
% The perturbation $\mathcal{E}(\delta)$ stems from non-zero correlation between the current tangent space and the prior manifold. Under \emph{Bounded Interference}, this correlation is suppressed; as $x^*$ moves away from the prior distribution, \Cref{lemma:asymptotic-decay} ensures $\Theta(y_k, y_{k-1}) \to 0$. This decouples the optimization of $\mathcal{R}_k$ from the functional form of $\mathcal{R}_{k-1}$. 
% As $t \to \infty$, the update converges to the kernel regression solution:
% \begin{equation}
%     y_k = y_{k-1} + \Theta_k(y_{k-1}, X) \Theta_k(X, X)^{-1} (Y - y_{k-1}) + O(\delta)
% \end{equation}
% The error propagates as $\|E_k\|_2 \le (e^{-\eta \lambda t} + C\delta)\|E_{k-1}\|_2$. In the limit $\delta \to 0$ and $t \to \infty$, this update simplifies to $y_k = y_{k-1} + \text{Proj}_{\Theta_k}(E_{k-1})$, which is identically the form of stage-wise gradient boosting in the function space.
% \end{proof}

\subsection{Antecedents of the DC3 Multi-ResNet Architecture}

\subsubsection{DC3 Hard Constraint Networks}

The DC3 (Deep Constraint Completion and Correction) framework is a differentiable approach designed to solve constrained optimization problems structured as $\min_w f_v(w)$ subject to inequality constraints $g_v(w) \leq 0$ and equality constraints $h_v(w) = 0$. 
By integrating these constraints directly into the neural network architecture, the method enforces equality constraints through completion and attempts to reduce inequality violations through unrolled correction, in two sequential stages.
\\

In the first stage, known as equality completion, the framework satisfies the equality constraints $h_v(w) = 0$ by partitioning the decision variable $w$ into independent variables $z$ and dependent variables $\phi_v(z)$. 
During the forward pass, the neural network $N_\theta$ predicts the independent component $z$, and a completion function $\mathcal{T}(z)$ solves for $\phi_v(z)$ such that the equality is maintained. 
To facilitate end-to-end training, the backward pass propagates gradients through this implicit solve using the Implicit Function Theorem, which defines the Jacobian as:
\begin{align}
\frac{\partial \phi_v(z)}{\partial z} = -\left[ \frac{\partial h_v}{\partial \phi_v} \right]^{-1} \frac{\partial h_v}{\partial z}.
\end{align}

The second stage, inequality correction, addresses the constraints $g_v(w) \leq 0$ through a differentiable procedure that adjusts the initial completed solution $\tilde{w}$. 
This is achieved by performing fixed gradient-based update steps to minimize the magnitude of any inequality violations:
\begin{align}
\hat{w} = \tilde{w} - \gamma \nabla_w \|\max(0, g_v(\tilde{w}))\|_2^2.
\end{align}
By unrolling these iterations during training, the network learns to predict an initial $z$ that requires minimal correction, effectively internalizing the feasible region's boundaries.
\\

The overall training procedure begins with the initialization of the neural network parameters $\theta$, defined as $N_\theta: \mathbb{R}^n \to \mathbb{R}^D$ where $D < d$. 
For each input $v$ in the dataset $\mathcal{V}$, the model proposes a latent vector $z = N_\theta(v)$. 
This proposal is then completed into a full solution $\tilde{w} = \mathcal{T}(z) = [z, \phi_v(z)]^T$ via the equality projection. 
Subsequently, the solution is corrected toward inequality feasibility via $\hat{w} = \Pi(\tilde{w})$.
A regularized loss $f(\hat{w} \mid v)$ is then computed, and the parameters $\theta$ are updated using the gradient $\nabla_\theta f(\hat{w} \mid v)$. 
This entire pipeline remains differentiable, allowing the model to optimize its weights based on the final, feasible output rather than a raw, unconstrained prediction.

\begin{algorithm}
\caption{DC3 Training}\label{alg:DC3train}
\begin{algorithmic}[1]
\Procedure{Train}{$\mathcal{V}$}
    \State Initialize $N_\theta : \R^n \to \R^D$ \Comment{$D < d$}
    \While{not converged}
        \For{$v \in \mathcal{V}$}
            \State Propose $z = N_\theta(v)$
            \State Complete $\tilde{w} = \mathcal{T}(z) =  \begin{bmatrix} z ,  \phi_v(z) \end{bmatrix}^T$ \Comment{Project to equality constraints}
            \State Correct $\hat{w} = \Pi(\tilde{w})$ \Comment{Project to inequality constraints}
            \State Compute regularized loss $f(\hat{w} \mid v)$
            \State Update $\theta$ using $\nabla_\theta f(\hat{w} \mid v)$
        \EndFor
    \EndWhile
\EndProcedure
\end{algorithmic}
\end{algorithm}

\subsubsection{U-Nets, Multi-ResNet and Conditioned ResNets}

\begin{figure}[h]
    \centering
    % --- Diagram 1: Standard U-Net ---
    \begin{minipage}{0.32\textwidth}
        \centering
        \scalebox{0.5}{
        \begin{tikzpicture}[
            node distance=1.0cm and 1.5cm,
            font=\sffamily,
            >=Stealth,
            block/.style={rectangle, draw=black, very thick, minimum width=2.5cm, minimum height=1.2cm, rounded corners=2pt},
            smallblock/.style={rectangle, draw=black, thick, minimum width=1.5cm, minimum height=0.8cm, rounded corners=2pt},
            skip/.style={dashed, ->, thick},
            line/.style={draw, ->, thick}
        ]
        % --- Level 2 (Top) ---
        \node[block] (E2) {$E_2$};
        \node[right=0.2cm of E2] {$V_2$};
        \node[above=0.8cm of E2] (input) {};
        \node[right=0.1cm of input] {$v_2$};
        \node[block, left=3cm of E2] (D2) {$D_2$};
        \node[left=0.2cm of D2] {$W_2$};
        \node[above=0.8cm of D2] (output) {$U_2(v_2)$};

        % --- Level 1 (Middle) ---
        \node[smallblock, below=1cm of E2] (P1) {$P_1$};
        \node[block, below=1cm of P1] (E1) {$E_1$};
        \node[right=0.2cm of E1] {$V_1$};
        \node[block, left=3cm of E1] (D1) {$D_1$};
        \node[left=0.2cm of D1] {$W_1$};

        % --- Level 0 (Bottom) ---
        \node[smallblock, below=1cm of E1] (P0) {$P_0$};
        \node[block, left=1.5cm of P0, anchor=east] (U0) {$U_0$};
        \node[below=0.1cm of U0] {Bottleneck};
        \node[right=0.1cm of U0, yshift=-0.3cm] {$V_0$};
        \node[left=0.1cm of U0, yshift=-0.3cm] {$W_0$};

        % --- Connections ---
        \draw[line] (input) -- (E2);
        \draw[line] (E2) -- (P1);
        \draw[line] (P1) -- (E1);
        \draw[line] (E1) -- (P0);
        \draw[line] (P0.west) |- (U0.east);
        \draw[line] (U0.west) -| (D1.south);
        \draw[line] (D1) -- (D2);
        \draw[line] (D2) -- (output);
        \draw[skip] (E2) -- (D2) node[midway, above] {\scriptsize Conditioning};
        \draw[skip] (E1) -- (D1) node[midway, above] {\scriptsize Conditioning};
        \end{tikzpicture}
        }
        \caption{General U-Net \cite{ronneberger2015unet, williams2023unified}}
    \end{minipage}
    \hfill
    % --- Diagram 2: Multi-ResNet (Identity Encoder) ---
    \begin{minipage}{0.32\textwidth}
        \centering
        \scalebox{0.5}{
        \begin{tikzpicture}[
            node distance=1.0cm and 1.5cm,
            font=\sffamily,
            >=Stealth,
            block/.style={rectangle, draw=black, very thick, minimum width=2.5cm, minimum height=1.2cm, rounded corners=2pt},
            smallblock/.style={rectangle, draw=black, thick, minimum width=1.5cm, minimum height=0.8cm, rounded corners=2pt},
            placeholder/.style={rectangle, draw=none, minimum width=2.5cm, minimum height=1.2cm},
            skip/.style={dashed, ->, thick},
            line/.style={draw, ->, thick}
        ]
        % --- Level 2 (Top) ---
        \node[placeholder] (E2) {};
        \node[right=0.2cm of E2] {$V_2$};
        \node[above=0.8cm of E2] (input) {};
        \node[right=0.1cm of input] {$v_2$};
        \node[block, left=3cm of E2] (D2) {$D_2$};
        \node[left=0.2cm of D2] {$W_2$};
        \node[above=0.8cm of D2] (output) {$U_2(v_2)$};

        % --- Level 1 (Middle) ---
        \node[smallblock, below=1cm of E2] (P1) {$P_1$};
        \node[placeholder, below=1cm of P1] (E1) {};
        \node[right=0.2cm of E1] {$V_1$};
        \node[block, left=3cm of E1] (D1) {$D_1$};
        \node[left=0.2cm of D1] {$W_1$};

        % --- Level 0 (Bottom) ---
        \node[smallblock, below=1cm of E1] (P0) {$P_0$};
        \node[block, left=1.5cm of P0, anchor=east] (U0) {$U_0$};
        \node[below=0.1cm of U0] {Bottleneck};
        \node[right=0.1cm of U0, yshift=-0.3cm] {$V_0$};
        \node[left=0.1cm of U0, yshift=-0.3cm] {$W_0$};

        % --- Connections ---
        \draw[line] (input) -- (E2.center);
        \draw[line] (E2.center) -- (P1);
        \draw[line] (P1) -- (E1.center);
        \draw[line] (E1.center) -- (P0);
        \draw[line] (P0.west) |- (U0.east);
        \draw[line] (U0.west) -| (D1.south);
        \draw[line] (D1) -- (D2);
        \draw[line] (D2) -- (output);
        \draw[skip] (E2.center) -- (D2) node[midway, above] {\scriptsize Conditioning};
        \draw[skip] (E1.center) -- (D1) node[midway, above] {\scriptsize Conditioning};
        \filldraw [black] (E2.center) circle (2pt);
        \filldraw [black] (E1.center) circle (2pt);
        \end{tikzpicture}
        }
        \caption{Multi-ResNet (U-Net specified with Identity Encoder)}
    \end{minipage}
    \hfill
    % --- Diagram 3: DC3 Multi-ResNet (Fixed V) ---
    \begin{minipage}{0.32\textwidth}
        \centering
        \scalebox{0.5}{
        \begin{tikzpicture}[
            node distance=1.0cm and 1.5cm,
            font=\sffamily,
            >=Stealth,
            block/.style={rectangle, draw=black, very thick, minimum width=2.5cm, minimum height=1.2cm, rounded corners=2pt},
            placeholder/.style={circle, fill=black, inner sep=1.5pt},
            skip/.style={dashed, ->, thick},
            line/.style={draw, ->, thick}
        ]
        % --- Level 2 (Top) ---
        \node[placeholder] (E2) {};
        \node[right=0.2cm of E2] {$V$};
        \node[above=0.8cm of E2] (input) {};
        \node[right=0.1cm of input] {$v$ (Input)};
        \node[block, left=3cm of E2] (D2) {$N_2(z_1|v)$};
        \node[left=0.2cm of D2] {$W_2$};
        \node[above=0.8cm of D2] (output) {$w_2$};

        % --- Level 1 (Middle) ---
        \node[placeholder, below=2.5cm of E2] (E1) {};
        \node[right=0.2cm of E1] {$V$};
        \node[block, left=3cm of E1] (D1) {$N_1(z_0|v)$};
        \node[left=0.2cm of D1] {$W_1$};

        % --- Level 0 (Bottom) ---
        \node[block, below left=2.5cm and 1.5cm of E1, anchor=east] (U0) {$N_0(v)$};
        \node[below=0.1cm of U0] {Base Anchor};
        \node[left=0.1cm of U0, yshift=-0.3cm] {$W_0$};

        % --- Connections ---
        \draw[line] (input) -- (E2);
        \draw[line] (E2) -- (E1) node[midway, right] {Id};
        \draw[line] (E1) |- (U0.east) node[near end, above] {Id};
        \draw[line] (U0.west) -| (D1.south);
        \draw[line] (D1) -- (D2);
        \draw[line] (D2) -- (output);
        \draw[skip] (E2) -- (D2) node[midway, above] {\scriptsize Cond. on $v$};
        \draw[skip] (E1) -- (D1) node[midway, above] {\scriptsize Cond. on $v$};
        \end{tikzpicture}
        }
        \caption{Conditioned ResNet (Multi-ResNet specified with identity projections)}
        \label{fig:U-multi-resnet}
    \end{minipage}
\end{figure}

Formally, this architecture instantiates a U-Net for subspace preconditioning by fixing the encoder and projection operators to the identity map, thereby preserving the full input state at every refinement stage.
This is the simplest form of U-Net.
Currently we do not do any feature compression, but this would need further investigation. 

\subsection{Gaussian Process Approximation}
\label{appendix:GP-Approximation}

Consider a deep fully-connected neural network $N_{\theta}$ of depth $L$. Let $w_l$ denote the width of layer $l$, and $\phi(\cdot)$ be a coordinate-wise activation function. To ensure the existence of a stable limit as $w \to \infty$, we adopt the Neural Network Tangent Kernel (NTK) parameterization. The pre-activations $z^{(l)}$ and activations $h^{(l)}$ are defined recursively for $l \in \{1, \dots, L\}$ via the relation:
\begin{equation}
z^{(l+1)}(v) = \frac{\sigma_w}{\sqrt{w_l}} W^{(l)} h^{(l)}(v) + \sigma_b b^{(l)}, \quad h^{(l)}(v) = \phi(z^{(l)}(v))
\end{equation}
where the weights and biases within $\theta$ are initialized as i.i.d. standard Gaussians, $W_{ij}^{(l)} \sim \mathcal{N}(0, 1)$ and $b_i^{(l)} \sim \mathcal{N}(0, 1)$. The scaling factor $1/\sqrt{w_l}$ ensures that the variance of the pre-activations remains $O(1)$, preventing internal signals from vanishing or exploding as hidden layers grow infinitely wide.

As layer widths approach infinity, the behavior of the network simplifies; by the Central Limit Theorem, the pre-activations at each layer converge to a Gaussian Process. This implies that $N_{\theta}$, while non-linear in its parameters $\theta$, behaves as a linear model near its initialization $\theta_0$: $N_{\theta}(v) \approx N_{\theta_0}(v) + \nabla_\theta N_{\theta_0}(v)^\top (\theta - \theta_0)$. Furthermore, the Neural Tangent Kernel $\Theta^{(L)}(v, v') = \nabla_\theta N_{\theta}(v)^\top \nabla_\theta N_{\theta}(v')$ remains constant during training such that $\Theta_t \approx \Theta_0$. This lazy training regime ensures that the inductive bias established by the subspace preconditioning architecture is preserved, as the governing kernel does not deviate from its initial state.

At initialization ($t=0$), the network converges in distribution to a zero-mean Gaussian Process, $N_{\theta, 0}(\cdot) \xrightarrow{d} \mathcal{GP}(0, K^{(L)})$, where $K^{(L)}$ is the Neural Network Gaussian Process (NNGP) kernel, computed recursively: $K^{(l+1)}(v, v') = \sigma_w^2 \mathbb{E}_{z \sim \mathcal{N}(0, \Sigma^{(l)})} [\phi(z_1)\phi(z_2)] + \sigma_b^2$. Under gradient descent with squared loss, the evolution of predictions $N_{\theta, t}$ for a test point $v$ given training data $(\mathcal{V}, \mathcal{W})$ follows:
\begin{equation}
\mu_t(v) = \Theta(v, \mathcal{V}) \Theta(\mathcal{V}, \mathcal{V})^{-1} (I - e^{-\eta \Theta(\mathcal{V}, \mathcal{V})t}) \mathcal{W}
\end{equation}
This confirms that in the infinite-width limit, the network is equivalent to kernel regression. By characterizing subspace preconditioning updates here, we treat recursive residual steps as deterministic kernel mappings, ensuring stable structural inductive bias.

\subsubsection{Extension to Graph Neural Networks (GNNs)}
\label{appendix:GNN-GP-Approximation}

The infinite-width limit can be extended to Graph Neural Networks, specifically Message Passing Neural Networks (MPNNs). Let $G = (\mathcal{V}, \mathcal{E})$ be a graph where each node $u \in \mathcal{V}$ has an initial feature vector $v_u$. A GNN layer $l$ updates node representations $h_u^{(l)}$ by aggregating features from the neighborhood $\mathcal{N}(u)$. Following the GNTK parameterization, the message passing and update steps are defined as:
\begin{align}
    a_u^{(l)} = \text{AGGREGATE}^{(l)}\left(\left\{h_j^{(l-1)} : j \in \mathcal{N}(u)\right\}\right), \quad z_u^{(l)} = \frac{\sigma_w}{\sqrt{n_l}} W^{(l)} a_u^{(l)} + \sigma_b b^{(l)}
\end{align}
where $W^{(l)}$ and $b^{(l)}$ are initialized as i.i.d. standard Gaussians. The choice of the aggregation function affects the scaling and resulting kernel structure. As the number of hidden channels $n_l \to \infty$, the representation of each node converges to a Gaussian Process. Unlike standard MLPs, the covariance between two graphs $G$ and $G'$ depends on the structural alignment of their nodes. The Graph Gaussian Process (GGP) kernel $K_{GGP}(G, G')$ is computed via a recursive neighborhood aggregation of the base NNGP kernels:
\begin{align}
    K_{GGP}^{(l)}(u, j) = \sigma_w^2 \mathbb{E}_{z \sim \mathcal{N}(0, \Sigma_{uj}^{(l-1)})} [\phi(z_u)\phi(z_j)] + \sigma_b^2
\end{align}
where $\Sigma_{uj}^{(l-1)}$ captures the covariance between nodes $u \in G$ and $j \in G'$ from the previous layer. Consequently, the global graph representation $f(G; \theta)$ converges in distribution to a Gaussian Process, $f(G; \theta) \xrightarrow{d} \mathcal{GP}(0, \Theta_{GNTK}(G, G'))$, where $\Theta_{GNTK}$ is the Graph Neural Tangent Kernel. This kernel describes GNN training dynamics under gradient descent and is computed by accumulating the contributions of the GGP kernel and its derivatives across message-passing steps:
\begin{align}
    \Theta_{GNTK}^{(l)}(G, G') = \sum_{u \in \mathcal{V}, j \in \mathcal{V}'} \Theta^{(l)}(u, j)
\end{align}
This allows for the exact characterization of GNN generalization on graph-structured data through kernel regression, effectively treating the GNN as a linear model over a structural feature space.

\subsection{Additional Experiments}
\label{sec:appendix-experiments}

\subsubsection{DC3 Correction Sensitivity Analysis}
\label{sec:appendix-dc3-analysis}

On line-flow-constrained ACOPF, vanilla DC3 does not realize its intended predict-and-correct division of labor.
We analyze DC3's correction effectiveness at $\alpha_S{=}0.5$ (the most congested regime) by comparing post-completion / pre-correction and post-correction violations at different correction learning rates.

\begin{table}[H]
\centering
\caption{DC3 correction effectiveness ($\gamma{=}0.01$, $\alpha_S{=}0.5$). T2 reduction percentage shown in the Regime column. At weak correction ($\leq 10^{-5}$), correction is inert (0--0.3\% T2 reduction). At active correction ($10^{-4}$), correction reduces T2 by 43\% but drifts off the equality manifold (Eq=0.151).}
\label{tab:dc3-sensitivity}
\begin{tabular}{lcccccl}
\hline
$\eta_{\text{corr}}$ & Pre-T1 & Post-T1 & Pre-T2 & Post-T2 & Post-Eq & Regime \\ \hline
$10^{-7}$ & 0.0081 & 0.0081 & 1.486 & 1.486 & 0 & Inert (0\%) \\
$10^{-5}$ & 0.0070 & 0.0086 & 1.486 & 1.482 & 0 & Inert (0.3\%) \\
$10^{-4}$ & 1.894 & 0.022 & 2.453 & 1.407 & \textbf{0.151} & Drifts (43\%) \\ \hline
\end{tabular}
\end{table}

At weak correction ($\leq 10^{-5}$), the correction operator does essentially nothing---the network alone produces the final answer, and DC3's intended division of labor is not operating.
At the active correction rate ($10^{-4}$, the default in the released DC3 codebase for ACOPF), correction reduces T2 by 43\%, but the network has learned to produce poor starting points (Pre-T2=2.45 vs.\ 1.49 at weak lr) and the correction drifts off-manifold (Eq=0.151).

DC3's equality drift also scales with congestion at the active learning rate: Eq = $0.010$ ($\alpha_S{=}1.0$) $\to$ $0.050$ ($\alpha_S{=}0.7$) $\to$ $0.152$ ($\alpha_S{=}0.5$).
This confirms that DC3 has no usable active-correction regime on nonlinear ACOPF with thermal line flow constraints.

\subsubsection{DC3+Manifold-Aware Ablation}
\label{sec:appendix-dc3-ma}

DC3+manifold\_aware projects correction gradients onto the tangent space of the equality manifold before each step, but does not re-project the solution back onto the manifold.
Its drift behavior depends on the objective weight $\gamma$:

\begin{table}[H]
\centering
\caption{DC3+manifold-aware equality drift. (a) At $\gamma{=}0.01$ (aggressive correction), MA makes drift worse. (b) At $\gamma{=}1.0$ (gentle correction), MA reduces drift.}
\label{tab:dc3-ma}
\begin{subtable}[t]{0.45\textwidth}
\centering
\caption{$\gamma = 0.01$ (MA worse)}
\begin{tabular}{lcc}
\hline
$\alpha_S$ & DC3 & DC3+MA \\ \hline
1.0 & 0.0098 & 0.3189 \\
0.7 & 0.0501 & 0.4447 \\
0.5 & 0.1523 & 0.5140 \\ \hline
\end{tabular}
\end{subtable}
\hfill
\begin{subtable}[t]{0.45\textwidth}
\centering
\caption{$\gamma = 1.0$ (MA better)}
\begin{tabular}{lcc}
\hline
$\alpha_S$ & DC3 & DC3+MA \\ \hline
1.0 & 0.8556 & 0.2198 \\
0.7 & 0.9860 & 0.0129 \\
0.5 & 0.8791 & 0.1060 \\ \hline
\end{tabular}
\end{subtable}
\end{table}

At $\gamma{=}1.0$ (gentle corrections), the manifold-aware gradient is a good approximation and helps reduce drift.
At $\gamma{=}0.01$ (aggressive corrections, used in our main experiments), gradient mismatch compounds across steps and makes drift worse.
This confirms that tangent-space projection alone is insufficient for manifold preservation on nonlinear ACOPF.

\subsubsection{IEEE 57-Bus Generalization}
\label{sec:appendix-57bus}

To confirm the drift finding is not specific to the 30-bus system, we evaluate DC3 and MResOpt on IEEE 57-bus ACOPF (128 vars, 114 eq, 302 ineq) with PGLib-calibrated thermal limits \cite{babaeinejadsarookolaee2019pglib}.
The original DC3/MATPOWER case57 has dummy $\alpha_S = 99.0$ (i.e., no binding thermal limits); we inject PGLib's branch ratings (25--1617 MVA across 80 branches).
Same settings: $\gamma{=}0.01$, $\eta_{\text{corr}}$=$10^{-4}$, 3 seeds.

Feasibility by $\alpha_S$: 1.0 = 100\%, 0.5 = 99\%, 0.3 = 6\%.

\begin{table}[H]
\centering
\caption{IEEE 57-bus ACOPF with PGLib thermal limits (3 seeds, $\eta_{\text{corr}}$=$10^{-4}$, $\gamma{=}0.01$).}
\label{tab:57bus}
\begin{tabular}{llcccc}
\hline
$\alpha_S$ & Method & OG (\%) & Equality (p.u.) & Tier-1 (p.u.) & Tier-2 (p.u.) \\ \hline
1.0 & DC3 & $+0.98 \pm 0.01$ & 0.0011 & $0.0057 \pm 0.000$ & $0.0000$ \\
    & DC3+recomp & $+0.93 \pm 0.08$ & $8.33\text{e-}08$ & $0.0077 \pm 0.000$ & $0.0000$ \\
    & MResOpt & $+1.32 \pm 0.07$ & $1.81\text{e-}05$ & $\mathbf{0.0039 \pm 0.000}$ & $0.0000$ \\ \hline
0.5 & DC3 & $3.82 \pm 0.00$ & 0.0011 & $0.0060 \pm 0.000$ & $\mathbf{0.0005 \pm 0.000}$ \\
    & DC3+recomp & $3.82 \pm 0.00$ & $7.20\text{e-}08$ & $0.0072 \pm 0.000$ & $0.0006 \pm 0.000$ \\
    & MResOpt & $3.84 \pm 0.01$ & $8.28\text{e-}06$ & $\mathbf{0.0042 \pm 0.000}$ & $0.0006 \pm 0.000$ \\ \hline
0.3 & DC3 & $3.90 \pm 0.00$ & 0.0118 & $0.0252 \pm 0.000$ & $\mathbf{0.1034 \pm 0.001}$ \\
    & DC3+recomp & $3.89 \pm 0.01$ & $3.02\text{e-}07$ & $0.0238 \pm 0.001$ & $0.1082 \pm 0.002$ \\
    & MResOpt & $3.96 \pm 0.01$ & $1.62\text{e-}05$ & $\mathbf{0.0177 \pm 0.001}$ & $0.1186 \pm 0.002$ \\ \hline
\end{tabular}
\end{table}

DC3 drifts on 57-bus too (Eq = 0.0011 at $\alpha_S{=}1.0$, scaling to 0.0118 at $\alpha_S{=}0.3$), while MResOpt stays on-manifold (Eq $\sim 10^{-5}$).
MResOpt wins Tier-1 at every congestion level (1.4--1.5$\times$).
The same pattern as 30-bus confirms that the drift phenomenon generalizes across bus systems with independently calibrated thermal limits.
DC3's drift is smaller on 57-bus than 30-bus (0.001--0.012 vs.\ 0.010--0.152), likely because PGLib's 57-bus thermal limits are less tight relative to the operating point.

\paragraph{DC3+recomp ablation.} Adding Newton power flow re-completion after each correction step (DC3+recomp) collapses the equality drift by 4--5 orders of magnitude relative to plain DC3 across all congestion levels (e.g., $1.18 \times 10^{-2} \to 3.02 \times 10^{-7}$ at $\alpha_S{=}0.3$), confirming that re-completion is the ingredient that restores manifold satisfaction. The effect on Tier-1 is more subtle and congestion-dependent: at low congestion ($\alpha_S{=}1.0/0.5$), re-completion actually \emph{worsens} Tier-1 over plain DC3 ($0.0057 \to 0.0077$ and $0.0060 \to 0.0072$, a $20$--$35\%$ increase), because re-projecting onto the equality manifold after each step pushes the iterate more firmly against the box bounds, whereas plain DC3's drift implicitly cushions those violations. Only in the heavily congested regime ($\alpha_S{=}0.3$) does re-completion help Tier-1 ($0.0252 \to 0.0238$, a $5.5\%$ reduction), where the off-manifold drift had begun to compound into the box hits. By contrast, MResOpt's staged tier decomposition wins Tier-1 \emph{unambiguously at every $\alpha_S$} ($1.3$--$2.0\times$ smaller than DC3+recomp). The ablation thus isolates two complementary ingredients with different roles: re-completion is required for Eq satisfaction but does not consistently improve Tier-1, while the staged tier decomposition is what produces the structural Tier-1 gain. MResOpt achieves both --- sub-$10^{-5}$ Eq drift \emph{and} the lowest Tier-1 across every congestion level. Optimality gap is comparable across all three methods at each $\alpha_S$.

\subsubsection{3-Bus ACOPF: Controlled Example and Ordering Ablation}
\label{sec:appendix-ordering-ablation}

We construct a minimal 3-bus ACOPF system (2 generators, 3 branches) to visualize the nested feasibility structure and ablate the constraint ordering.
Figure~\ref{fig:3bus_3d_appendix} shows the maximum line flow as a surface over the generator dispatch space $(P_{g_2}, V_{m_2})$, with horizontal planes at three thermal limits $\bar{S}$.
As $\bar{S}$ decreases, the feasible region $W_3$ shrinks and eventually vanishes ($\bar{S} = 38$ MVA), illustrating the overconstrained regime.

\begin{figure}[H]
\centering
\includegraphics[width=0.7\textwidth]{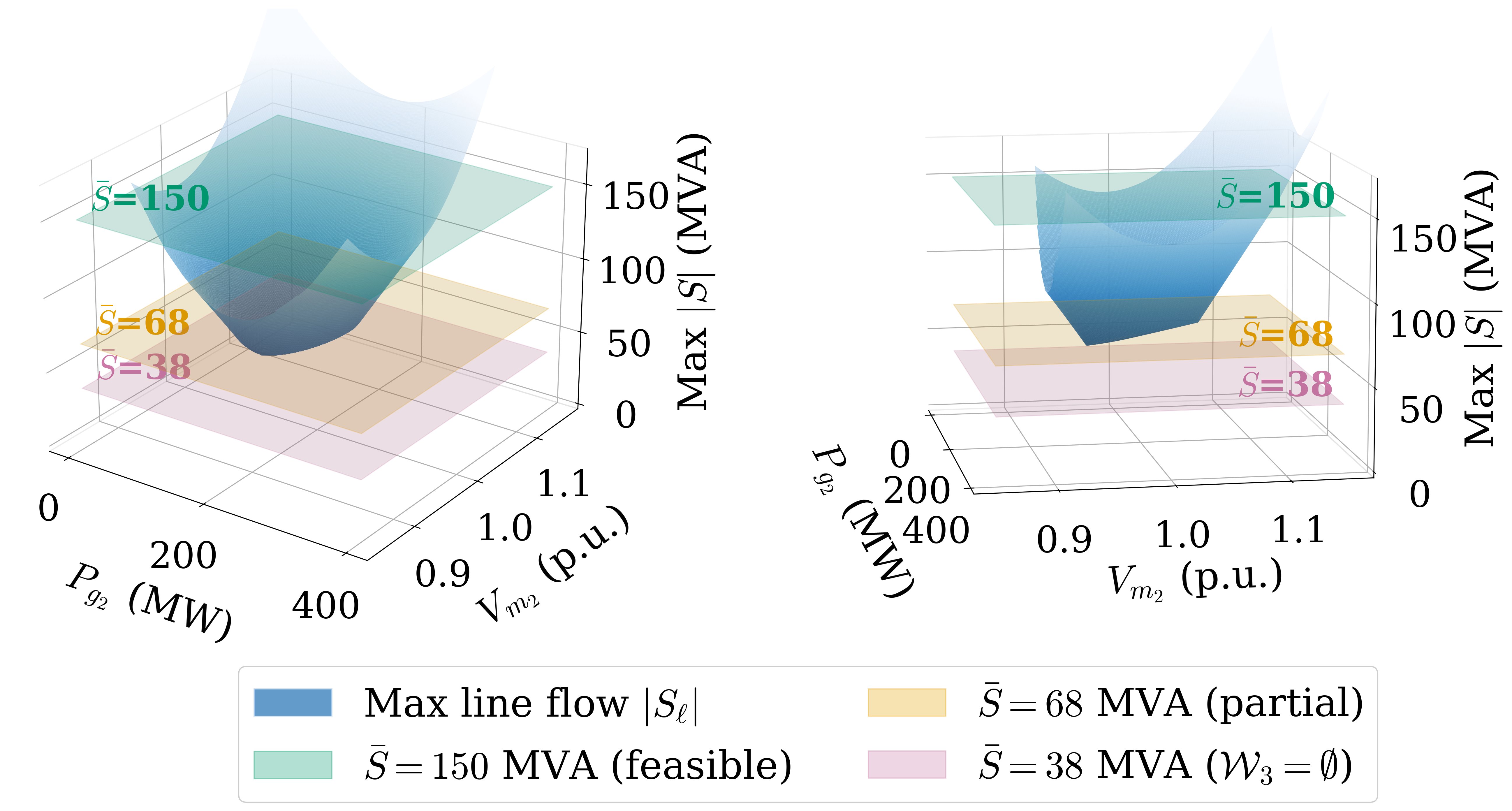}
\caption{3-bus ACOPF feasibility. \textbf{Left}: front view. \textbf{Right}: side view. Blue surface: max line flow $|S_\ell|$ over $(P_{g_2}, V_{m_2})$. Horizontal planes show thermal limits: $\bar{S}=150$ MVA (teal) intersects ($W_3$ exists); $\bar{S}=68$ MVA (amber) leaves a small feasible region; $\bar{S}=38$ MVA (pink) is infeasible ($W_3=\emptyset$).}
\label{fig:3bus_3d_appendix}
\end{figure}

We train DC3, MResOpt, and MResOpt-det on this system with tightened bounds ($P_g \in [80, 200]$ MW, $Q_g \in [-30, 80]$ MVAr, $V_m \in [0.95, 1.05]$ p.u.) at $\bar{S} = 38$ MVA (fully infeasible for line flows).
All models use 2 layers, 16 hidden units, $\eta_{\text{corr}} = 10^{-4}$, 500 epochs, 2000 training samples.

\begin{table}[H]
\centering
\caption{3-bus ACOPF results (mean $\pm$ std over 3 seeds). Equality violations are near-zero for all methods ($\sim\!10^{-7}$). MResOpt-det preserves Tier-1 $3.1\times$ better than DC3.}
\label{tab:3bus_appendix}
\begin{tabular}{lccc}
\hline
Method & T1 (box) & T2 (flow) \\ \hline
DC3 & $0.0054 \pm 0.0012$ & $\mathbf{0.558 \pm 0.003}$ \\
MResOpt & $0.0021 \pm 0.0002$ & $0.557 \pm 0.003$ \\
MResOpt-det & $\mathbf{0.0017 \pm 0.0002}$ & $0.574 \pm 0.003$ \\ \hline
\end{tabular}
\end{table}

\paragraph{Ordering ablation.}
To test whether the constraint ordering matters, we reverse the tier assignment: Tier-1 = line flows, Tier-2 = box bounds (opposite of the physics-motivated ordering).

\begin{table}[H]
\centering
\caption{3-bus ACOPF with reversed tier ordering (T1 = line flows, T2 = box bounds). All methods show T1 $\approx 0.555$ --- no separation. The reversed ordering eliminates MResOpt's advantage because the ``easy'' tier (box bounds) is now in Stage~2 while the impossible constraint (line flows at $\bar{S} = 38$ MVA) is in Stage~1.}
\label{tab:3bus_reversed}
\begin{tabular}{lccc}
\hline
Method & T1 (flow) & T2 (box) & Eq \\ \hline
DC3 & 0.554 & 0.005 & $4.8 \times 10^{-5}$ \\
MResOpt & 0.555 & 0.016 & $5.8 \times 10^{-6}$ \\
MResOpt-det & 0.555 & 0.242 & $1.45$ (diverges) \\ \hline
\end{tabular}
\end{table}

With the reversed ordering, all methods achieve T1 $\approx 0.555$ (identical line flow violation) --- the hierarchy provides no benefit because Stage~1 is asked to solve the infeasible constraint first.
Moreover, MResOpt-det diverges entirely (Eq = 1.45), confirming that the detach mechanism is sensitive to the ordering: blocking gradient flow from an infeasible tier is harmful when that tier is placed first.
This ablation supports the domain-knowledge-driven ordering: placing simpler, solvable constraints (box bounds) in Stage~1 gives the correction a smooth landscape to operate on, while deferring complex constraints (line flows) to Stage~2.

\subsubsection{DC3's Original ACOPF Setting}
\label{sec:appendix-dc3-easy}

The original DC3 paper \cite{donti2021dc3} evaluated on IEEE 57-bus ACOPF using only generator and voltage box bounds (no thermal line-flow limits).
In this setting, all methods perform near-identically:

\begin{table}[H]
\centering
\caption{57-bus ACOPF without line flows (DC3's original setting). All methods achieve near-identical OG and small violations.}
\label{tab:dc3-easy}
\begin{tabular}{lcccc}
\hline
Method & Eq & Gen-bound & Voltage-bound & OG \\ \hline
DC3 & $5.0 \times 10^{-11}$ & $3.0 \times 10^{-3}$ & $6.2 \times 10^{-3}$ & $-0.10\%$ \\
MResOpt & $5.5 \times 10^{-8}$ & $7.9 \times 10^{-4}$ & $1.9 \times 10^{-2}$ & $-0.09\%$ \\
MResOpt-det & $1.1 \times 10^{-7}$ & $2.2 \times 10^{-3}$ & $7.9 \times 10^{-2}$ & $-0.09\%$ \\ \hline
\end{tabular}
\end{table}

Without line flow constraints, the constraint landscape is smooth enough for a single correction pass --- DC3's correction is sufficient because the feasible region does not exhibit the complex, non-smooth geometry that creates basin trapping on the full ACOPF.
Adding 82 line flow constraints (as in our main experiments) creates the nonlinear tension where DC3's correction drifts off the equality manifold.

\subsubsection{DCOPF and ACOPF Without Line Flows}
\label{sec:appendix-easy-problems}

Two ablations confirm that the staged mechanism provides its advantage when the constraint landscape is complex and non-smooth, creating multiple basins that a single correction pass cannot navigate. When the landscape is simple, decomposition into subproblems offers no benefit.

\paragraph{DCOPF (IEEE 30-bus, linear).} DCOPF is a linear QP (35 vars, 82 ineq, 30 eq). T1 = generator bounds, T2 = line flow limits.

\begin{table}[H]
\centering
\caption{DCOPF results: all methods achieve T1=0. The smooth, convex landscape admits a single correction pass without basin trapping.}
\label{tab:dcopf}
\begin{tabular}{llcccc}
\hline
Setting & Method & Eq & T1 & T2 & OG \\ \hline
Congested ($\alpha_S{=}0.3$) & DC3 & $2.3 \times 10^{-14}$ & $0.0$ & $7.6 \times 10^{-3}$ & $+0.02\%$ \\
 & MResOpt & $1.7 \times 10^{-14}$ & $0.0$ & $5.6 \times 10^{-3}$ & $-0.02\%$ \\
 & MResOpt-det & $1.6 \times 10^{-14}$ & $0.0$ & $4.0 \times 10^{-3}$ & $-0.03\%$ \\ \hline
Regular ($\alpha_S{=}1.0$) & DC3 & $1.3 \times 10^{-14}$ & $0.0$ & $0.0$ & $0.0\%$ \\
 & MResOpt & $1.3 \times 10^{-14}$ & $0.0$ & $0.0$ & $0.0\%$ \\
 & MResOpt-det & $1.3 \times 10^{-14}$ & $0.0$ & $0.0$ & $0.0\%$ \\ \hline
\end{tabular}
\end{table}

When the constraint landscape is smooth and convex, a single correction pass suffices --- the staged decomposition provides no additional benefit.

\subsection{Implementation Details}
\label{sec:appendix-implementation}

\paragraph{Architecture.}
DC3 uses an MLP with 1024 hidden units, 4 layers, SiLU activation, and sigmoid output ($\sim$3.25M parameters).
MResOpt uses a two-stage residual MLP: Stage~1 has 850 hidden units, 3 layers, sigmoid output; Stage~2 has 850 hidden units, 3 layers, tanh activation for the residual. The composition is $z_2 = \sigma(\text{logit}(z_1) + 0.1 \cdot \delta)$ where $\delta$ is Stage~2's output, initialized to zero ($\sim$3.11M parameters).

\paragraph{Training.}
AdamW optimizer with weight decay $10^{-3}$.
Learning rate: $5 \times 10^{-4}$ (synthetics), $2 \times 10^{-4}$ (ACOPF), with StepLR decay.
Batch size: 512.
Epochs: 500 (synthetics), 200--300 (ACOPF).
ACOPF correction: $\eta_{\text{corr}} = 10^{-4}$, momentum 0.5, 5+5 steps (Stage~1 + Stage~2).

\paragraph{Datasets.}
Synthetics: $d = 100$ variables, $m_{\text{eq}} = 50$, $m_{\text{ineq}} = 100$, 10{,}000 samples (following FSNet \cite{nguyen2025fsnet}).
ACOPF: IEEE 30-bus PGLearn benchmark \cite{klamkin2025pglearn}, 72 variables, 60 equality constraints, 166 inequality constraints (84 box + 82 line flow), 10{,}000 samples split 7{,}000/1{,}000/2{,}000 (train/val/test).

\paragraph{Compute.}
All experiments were run on a single NVIDIA A100 80GB GPU. Synthetic problems train in roughly 20 minutes per run (500 epochs, $\sim$2.5s/epoch). IEEE 30-bus ACOPF training takes approximately 50 minutes per run (300 epochs, $\sim$10s/epoch). IEEE 57-bus ACOPF training time depends on the method: vanilla DC3 finishes in a few minutes ($\sim$2.5s/epoch over 100 epochs), while DC3+recomp takes approximately 3 hours per run ($\sim$60s/epoch over 200 epochs), because DC3+recomp performs 10 Newton power-flow calls per sample versus 1 for DC3. Per-sample inference times are reported in Section 5.2.

\subsection{Synthetic Problem Formulations}
\label{sec:appendix-synthetic-formulations}

We use the synthetic problem suite from FSNet \cite{nguyen2025fsnet}. Each problem has $n = 200$ input parameters and $d = 100$ decision variables with bounds $y \in [-5, 5]^d$.

\paragraph{Convex problems.}
\begin{align}
\textbf{QP:}\quad &\min_y \tfrac{1}{2} y^\top Q y + p^\top y \quad \text{s.t.}\quad Ay = x,\; Gy \le h,\; L \le y \le U \\
\textbf{QCQP:}\quad &\min_y \tfrac{1}{2} y^\top Q y + p^\top y \quad \text{s.t.}\quad Ay = x,\; y^\top H_i y + g_i^\top y \le h_i,\; L \le y \le U \\
\textbf{SOCP:}\quad &\min_y \tfrac{1}{2} y^\top Q y + p^\top y \quad \text{s.t.}\quad Ay = x,\; \|G_i y + h_i\|_2 \le c_i^\top y + d_i,\; L \le y \le U
\end{align}

\paragraph{Nonconvex problems.} Constructed by applying element-wise sine/cosine transforms:
\begin{align}
\textbf{NC-QP:}\quad &\min_y \tfrac{1}{2} y^\top Q y + p^\top \sin(y) \quad \text{s.t.}\quad Ay = x,\; G \sin(y) \le h \cos(x),\; L \le y \le U
\end{align}
NC-QCQP and NC-SOCP follow analogously (see FSNet \cite{nguyen2025fsnet} Appendix A.2).

\paragraph{Tier split.}
Tier-1 contains the first 50 main inequalities plus all $2d = 400$ box bound constraints.
Tier-2 contains the remaining 50 main inequalities.
Box bounds are always in Tier-1 because they represent hard variable feasibility limits.

\end{document}